\documentclass[twoside]{article}

\usepackage[accepted]{aistats2021}

\usepackage[round]{natbib}

\usepackage[utf8]{inputenc}
\usepackage{amsmath,amssymb,amsfonts,mathrsfs}

\usepackage{graphicx}
\graphicspath{./figures}

\usepackage{varioref}

\usepackage{datetime}

\usepackage{chngcntr}

\usepackage{mathtools}

\usepackage[h]{esvect}

\usepackage{array}

\usepackage[activate]{pdfcprot}

\usepackage{booktabs}

\usepackage{hyperref}

\usepackage{calc}  
\usepackage{enumitem}

\usepackage{amsmath}

\usepackage{bm}

\usepackage[ruled,vlined]{algorithm2e}

\usepackage{caption}
\usepackage{subcaption}

\usepackage{float}

\usepackage[amsmath,thmmarks]{ntheorem}

\newtheorem{theorem}{Theorem}[section]

\newtheorem{proposition}[theorem]{Proposition}

\theoremstyle{nonumberplain}
\theorembodyfont{\normalfont}
\theoremsymbol{\ensuremath{\square}}
\newtheorem{proof}{Proof}

\usepackage[dvipsnames]{xcolor}

\usepackage{xcolor}

\newcommand{\N}{\mathcal{N}}

\newcommand{\R}{\mathbb{R}}

\newcommand{\E}{\mathbb{E}}

\newcommand{\bigO}{\mathcal{O}}
\newcommand{\cL}{\mathcal{L}}

\DeclarePairedDelimiter\norm{\lVert}{\rVert}

\renewcommand{\epsilon}{\ensuremath\varepsilon}

\newcommand{\zb}{\textbf{z}} 
\newcommand{\xb}{\textbf{x}}
\newcommand{\yb}{\textbf{y}} 
\newcommand{\fb}{\textbf{f}}
\newcommand{\ub}{\textbf{u}}

\newcommand{\ybt}{\tilde{\yb}}
\newcommand{\Ab}{\textbf{A}}
\newcommand{\Fb}{\textbf{F}}
\newcommand{\mub}{\bm{\mu}}
\newcommand{\sigbt}{{\tilde{\bm{\sigma}}}}
\newcommand{\sigbtt}{{\tilde{\bm{\sigma}}}^{-2}}

\newcommand{\Zb}{\textbf{Z}} 
\newcommand{\Xb}{\textbf{X}}
\newcommand{\Yb}{\textbf{Y}}
\newcommand{\Kb}{\textbf{K}} 
\newcommand{\Ib}{\textbf{I}}
\newcommand{\Ub}{\textbf{U}}
\newcommand{\Pb}{\textbf{P}}
\newcommand{\Lb}{\textbf{L}}
\newcommand{\Vb}{\textbf{V}}

\newcommand{\ELBO}{\cL}
\newcommand{\KL}{\text{KL}}
\newcommand{\diag}{\text{diag}}

\newcommand{\ppsi}{p_{\psi}}

\begin{document}

\twocolumn[

\aistatstitle{Scalable Gaussian Process Variational Autoencoders}

\aistatsauthor{ Metod Jazbec\footnotemark[1] \And Matthew Ashman \And Vincent Fortuin}

\aistatsaddress{ ETH Z\"urich \And University of Cambridge \And  ETH Z\"urich}

\aistatsauthor{ Michael Pearce \And Stephan Mandt \And Gunnar R\"atsch }

\aistatsaddress{ University of Warwick \And University of California, Irvine \And ETH Z\"urich}

\runningauthor{Jazbec, Ashman, Fortuin, Pearce, Mandt, and R\"atsch}
]

\begin{abstract}
Conventional variational autoencoders fail in modeling correlations between data points due to their use of factorized priors.
  Amortized Gaussian process inference through GP-VAEs has led to significant improvements in this regard, but is still inhibited by the intrinsic complexity of exact GP inference.
  We improve the scalability of these methods through principled sparse inference approaches.
  We propose a new scalable GP-VAE model that outperforms existing approaches in terms of runtime and memory footprint, is easy to implement, and allows for joint end-to-end optimization of all components.
\end{abstract}

\section{Introduction}

Variational autoencoders (VAEs) are among the most widely used models in representation learning and generative modeling \citep{Kingma2014Auto-encodingBayes,kingma2019introduction, rezende2014stochastic}. As VAEs typically use factorized priors, they fall short when modeling correlations between different data points. However, more expressive priors that capture correlations enable useful applications. \citet{Casale2018GaussianAutoencoders}, for instance, showed that by modeling prior correlations between the data, one could generate a digit's rotated image based on rotations of the same digit at different angles. 

Gaussian process VAEs (GP-VAEs) have been designed to overcome this shortcoming \citep{Casale2018GaussianAutoencoders}. These models introduce a Gaussian process (GP) prior over the latent variables that correlates the latent variables through a kernel function. While GP-VAEs have outperformed standard VAEs on many tasks \citep{Casale2018GaussianAutoencoders, Fortuin2019GP-VAE:Imputation, Pearce2019ThePixels}, combining the GPs and VAEs brings along fundamental computational challenges. On the one hand, neural networks reveal their full power in conjunction with large datasets, making mini-batching a practical necessity. 
GPs, on the other hand, are traditionally restricted to medium-scale datasets due to their unfavorable scaling.  In GP-VAEs, these contradictory demands must be reconciled, preferably by reducing the  $\bigO(N^3)$ complexity of GP inference, where $N$ is the number of data points.

Despite recent attempts to improve the scalability of GP-VAE models by using specifically designed kernels and inference methods \citep{Casale2018GaussianAutoencoders, Fortuin2019GP-VAE:Imputation}, a generic way to scale these models, regardless of data type or kernel choice, has remained elusive. This limits current GP-VAE implementations to small-scale datasets. In this work, we introduce the first generically scalable method for training GP-VAEs based on inducing points. We thereby improve the computational complexity from $\bigO(N^3)$ to $\bigO(bm^2 + m^3)$, where $m$ is the number of inducing points and $b$ is the batch size. 

We show that applying the well-known inducing point approaches \citep{Hensman2013GaussianData,Titsias2009VariationalProcesses} to GP-VAEs is a non-trivial task: existing sparse GP approaches cannot be used off-the-shelf within GP-VAE models as they either necessitate having the entire dataset in the memory or do not lend themselves to being amortized. To address this issue, we propose a simple hybrid sparse GP method that is amenable to both
mini-batching and amortization.

\footnotetext[1]{Contact: \href{mailto:fortuin@inf.ethz.ch}{\texttt{jazbec.metod@gmail.com}}}

We make the following contributions:

\begin{itemize}
    \item We propose the first scalable GP-VAE framework based on sparse GP inference (Sec.~\ref{sec:methods}). In contrast to existing methods, our model is agnostic to the kernel choice, makes no assumption on the structure of the data at hand and allows for joint optimization of all model components.
    \item We provide theoretical motivations for the proposed method and introduce a hybrid sparse GP model that accommodates a crucial demand of GP-VAEs for simultaneous amortization and batching.
    \item We show empirically that the proposed approximation scheme maintains a high accuracy while being much more scalable and efficient (Sec.~\ref{sec:experiments}). Importantly from a practitioner's point of view, our model is easy to implement as it requires no special modification of the training procedure.  
\end{itemize}

\section{Related Work}

\paragraph{Sparse Gaussian processes.}

There has been a
long line of work on sparse Gaussian process approximations,
dating back to \cite{Snelson2005SparsePseudo-inputs}, \cite{Quinonero-Candela2005ARegression},
and others. Most of these sparse methods rely on a summarizing set of points referred to as \emph{inducing points} and mainly differ in the exact way of selecting those. Variational learning of inducing points was first considered in \cite{Titsias2009VariationalProcesses} and was shown to lead to significant performance gains. Instead of optimizing an approximate marginal GP likelihood as done in non-variational sparse models, a lower bound on the exact GP marginal likelihood is derived and used as a training objective. Another approach relevant for our work is the stochastic variational approach from \cite{Hensman2013GaussianData}, where the authors proposed a sparse model that can, in addition to reducing the GP complexity, also be trained in mini-batches, enabling the use of GP models on (extremely) large datasets.

\paragraph{Improving VAEs.}
Extending the expressiveness and representational power of VAEs can be roughly divided into two (orthogonal) approaches. The first one focuses on increasing the flexibility of the approximate posterior \citep{rezende2016variational, Kinga2019variational}, while the second one consists of imposing a richer prior distribution on the latent space. Various extensions to the standard Gaussian prior have been proposed, including a Gaussian mixture prior \citep{Dilokthanakul2016DeepUC, kopf2019mixture}, hierarchical structured priors \citep{johnson2016composing,deng2017factorized},  and a von Mises-Fisher distribution prior \citep{Davidson2018HypersphericalAuto-encoders}. GP-VAE models are part of this second group and, contrary to other work on extending VAE priors, aim to relax the \emph{iid} assumption between data points. Moreover, GP-VAEs are also related to approaches that aim to learn more structured and interpretable representations of the data by incorporating auxiliary information, such as time or viewpoints \citep{sohn2015learning, lin2018variational, johnson2016composing}.

\paragraph{Gaussian process VAEs.}

As mentioned above, the most related approaches to our work are the GP-VAE models of \citet{Casale2018GaussianAutoencoders} and \citet{Pearce2019ThePixels}.
However, neither of these are scalable for generic kernel choices and data types. The model from \cite{Pearce2019ThePixels} relies on exact GP inference, while \cite{Casale2018GaussianAutoencoders} exploit a (partially) linear structure of their GP kernel and use a Taylor approximation of the ELBO to get around computational challenges. Another GP-VAE model is proposed in \citet{Fortuin2019GP-VAE:Imputation} where it is used for multivariate time series imputation. Their model is indeed scalable (even in linear time complexity), but it works exclusively on time series data since it exploits the Markov assumption. Additionally, it does not support a joint optimization of GP parameters, but assumes a fixed GP kernel.

\section{Scalable SVGP-VAE}
\label{sec:methods}

This work's main contribution is the sparsification of the GP-VAE using the sparse GP approaches mentioned above. To this end, two separate variational approximation problems have to be solved jointly: an outer amortized inference procedure from the high-dimensional space to the latent space, and the inner sparse variational inference scheme on the GP. To motivate our proposed solution, we begin by pointing out the problems that arise when na\"ively combining the two objectives.

\subsection{Problem setting and notation}
In this work, we consider high-dimensional data 
$\Yb = [\yb_1, \dots, \yb_N]^\top \in \mathbb{R}^{N \times K}$. Each data point has a corresponding low-dimensional auxiliary data entry, summarized as 
$\Xb = [\xb_1, \dots, \xb_N]^\top \in \mathcal{X}^{N}, \mathcal{X}\subseteq \R^D$.
For example, $\yb_i$ could be a video frame and $\xb_i$ the corresponding time stamp. Our goal is to train a model for (1) generating  $\Yb$ conditioned on $\Xb$ and (2) infering an interpretable
and disentangled low-dimensional representations.
 
To this end, we adopt a latent GP approach, summarized below. First, we need to model a prior distribution over the collection of latent variables $\Zb = [\zb_1,\dots,\zb_N]^T\in\R^{N\times L}$, each latent variable $\zb_i$ living in an $L$-dimensional latent space. To model their joint distribution, we 
assume $L$ independent latent
functions $f^1,\dots,f^L \sim GP(0, \: k_{\theta})$
with kernel parameters
$\theta$ that result in $\Zb$ when being evaluated on $\Xb$. More precisely, $\zb_i = [f^1(\xb_i),\dots,f^L(\xb_i)]$. 
By construction, the $l^{th}$ latent channel of all latent variables $\zb^{l}_{1:N}\in\R^N$ (the $l^{th}$ column of $\Zb$) has a correlated Gaussian prior
with covariance $\Kb_{NN}=k_\theta(\Xb, \Xb)$. %
Setting $\Kb_{NN}=I$ recovers the fully factorized prior commonly used
in standard VAEs.

As in regular VAEs, each $\zb_i\in\R^L$ is then ``decoded" to parameterize the distribution over observations
$\yb_i = \mu_{\psi}(\zb_i) + \bm{\epsilon}_i$ where $\mu_\psi:\R^L\to\R^K$ is a network with parameters $\psi$ and $\bm{\epsilon}_i \sim \mathcal{N}(\bm{0}, \: \sigma_y^2 \: \Ib_K)$.
Mathematically, the full generative model is given by
\begin{align*}
    p_\theta(\Zb|\Xb) &= \prod_{l=1}^L\N(\zb_{1:N}^{l}| 0, \textbf{K}_{NN}), \\
    p_\psi(\Yb|\Zb) &= \prod_{i=1}^Np_\psi(\yb_i|\zb_i) = \prod_{i=1}^N\N(\yb_i| \mu_\psi(\zb_i), \sigma_y^2 \: \Ib_K).
\end{align*}
The joint distribution is $p_{\psi,\theta}(\Yb, \Zb|\Xb) = p_\psi(\Yb|\Zb)p_\theta(\Zb|\Xb)$.
The true posterior for the latent variables
$p_{\psi,\theta}(\Zb | \Yb, \Xb) = p_{\psi,\theta}(\Yb, \Zb|\Xb) / p_{\psi,\theta}(\Yb|\Xb)$ 
is intractable due to the denominator which requires integrating over $\Zb$. Hence, approximate
inference methods are required to infer the unobserved $\Zb$ given the observed $\Xb$ and $\Yb$.

\subsection{Amortized variational inference}

Amortization in the typical VAE architecture
uses a second (inference) network from the high-dimensional data $\yb_i$ to
the mean and variance
of a fully factorized Gaussian distribution over $\zb_i\in\R^L$ \citep{zhang2018advances}. We denote it as $\tilde q_\phi(\zb_i|\yb_i) = \mathcal{N}(\zb_i|\mu_\phi(\yb_i), \diag(\sigma^2_\phi(\yb_i)))$ 
and it has network parameters $\phi$. In 
\citet{Casale2018GaussianAutoencoders}, this Gaussian
distribution is used directly to approximate the posterior, $p_{\psi,\theta}(\Zb|\Yb)\approx \prod_i\tilde{q}_\phi(\zb_i|\yb_i)$. %
While this approach mirrors classical VAE design, the approximate posterior for
a latent variable $\zb_i$ only depends on $\yb_i$ and ignores $\xb_i$. This is
in stark contrast to traditional Gaussian processes where
latent function values
$f(x)$ are informed by all $y$ values according to the similarity of
the corresponding $x$ values.

Building on this model, \citet{Pearce2019ThePixels} instead 
proposed to use the inference network $\tilde q_\phi(\zb_i|\yb_i)$ to
replace only the intractable likelihood $p_\psi(\yb_i|\zb_i)$ in the posterior. By combining $\tilde q_{\phi}$ with tractable terms, the approximate posterior could be explicitly normalized as 
\begin{align}\label{eq:GP_reg_view}
    q(\Zb | \Yb, \Xb, \phi, \theta) := \prod_{l=1}^L\frac{\prod_{i=1}^N\Tilde{q}_{\phi}(\zb^{l}_{i} | \yb_i) \, p_{\theta}(\zb^{l}_{1:N} | \Xb)}{Z^{l}_{\phi,\theta}(\Yb, \Xb)},
\end{align}
where the normalizing constant $Z^{l}_{\phi,\theta}(\Yb, \Xb)$ can be computed
analytically. Noting the symmetry of the Gaussian distribution,
$\mathcal{N}(z| \mu, \sigma) = \mathcal{N}(\mu |z, \sigma)$,
the approximate posterior for channel $l$ is mathematically equivalent to the (exact) GP posterior in the
traditional GP regression with inputs $\Xb$
and outputs $\ybt_l:=\mu^l_\phi(\Yb)$ with heteroscedastic
noise $\sigbt_l := \sigma^l_\phi(\Yb)$. 
We therefore refer to each $\{\Xb, \ybt_l, \sigbt_l\}$ %
as the \emph{latent dataset} for the $l^{th}$ channel. Each normalizing constant of Equation~\ref{eq:GP_reg_view}
is also the GP marginal likelihood of the $l^{th}$ latent dataset.
The parameters $\{\psi, \phi, \theta\}$ are learnt by maximizing the
evidence lower bound (ELBO) in the \textbf{P}earce model,
\begin{align}\label{eq:PearceELBO}
 \cL_P(\psi,\phi, \theta) = \; \sum_{i=1}^N &\mathbb{E}_{q(\zb_i|\cdot)}  \bigg[  \log \ppsi(\yb_i | \zb_i) - \log \tilde{q}_\phi(\zb_i | \yb_i)\bigg] \nonumber \\ &+ \sum_{l=1}^L \log Z_{\phi, \theta}^{l}(\Yb, \Xb). 
\end{align}
The first term is the difference between the true likelihood
and inference network approximate likelihood, while the second term is
the sum over GP marginal likelihoods of each latent dataset. 

One subtle, yet important, characteristic of the variational approximation from \cite{Pearce2019ThePixels} is that it gives rise to the ELBO $ \cL_P(\cdot)$ that contains the GP posterior. Note that this is in contrast to  \cite{Casale2018GaussianAutoencoders} and \cite{Fortuin2019GP-VAE:Imputation}, where the GP prior is part of the ELBO. As we will show in Section \ref{sec:latent_sparse_GP}, the ELBO that contains the GP posterior naturally lends itself to "sparsification" through the use of sparse GP posterior approximations.

The computational challenges of $\cL_P(\cdot)$ are twofold.
Firstly, for the latent GP regression, an inverse
and a log-determinant of the kernel matrix 
$\Kb_{NN} \in \R^{N \times N}$ must be computed,
resulting in $\bigO(N^3)$ time complexity. Secondly, the
ELBO does not decompose as a sum over data points, so the
entire dataset $\{\Xb,\Yb\}$ is needed for one evaluation of $\cL_P(\cdot)$.

Given the latent dataset, at first glance, we may simply apply
sparse GP regression techniques instead of traditional regression. 
We next look at two widely used methods (\cite{Titsias2009VariationalProcesses} and \cite{Hensman2013GaussianData}) and highlight their drawbacks
for this task. We then propose a new hybrid approach solving these
issues.

\subsection{Latent Sparse GP Regression}
\label{sec:latent_sparse_GP}
To simplify the notation, we focus on a single channel and
suppress $l$, resulting in  $\ybt$ and $\bm{\tilde{\sigma}}$,
$\log Z_{\theta, \phi}(\cdot)$ and $f$.
Given an (amortized latent) regression dataset $\Xb,\, \ybt,\, \bm{\tilde{\sigma}}$,
sparse Gaussian process methods assume that there exists a set of $m\ll N$ 
inducing points with inputs $\Ub=[\ub_1,\dots,\ub_m]\in \mathcal{X}^m$ and outputs 
$\fb_m:=f(\Ub)\sim\mathcal{N}(f(\Ub)|\: \mub, \Ab)$ that summarize the regression dataset.
$\Ub,\, \mub,\, \Ab$ are parameters to be learnt. Given a (test) set
of $r$ new inputs $\Xb_r$, the sparse approximate (predictive) distribution
over outputs $\fb_r = f(\Xb_r)$ is
\begin{multline}%
    q_S(\fb_r|\Xb_r, \Ub, \mub, \Ab, \theta) =\\
    \mathcal{N}\big( \fb_r |\Kb_{rm} \Kb_{mm}^{-1} \mub, \: \Kb_{rr} -  \Kb_{rm}\Kb_{mm}^{-1}\Kb_{mr} \\
    \quad\quad\quad\quad\quad\quad\quad \quad \,\, +\Kb_{rm}\Kb_{mm}^{-1}\textbf{A} \Kb_{mm}^{-1} \Kb_{mr}\big),  \label{eq:svgp_approx_post}
\end{multline}
where kernel matrices are $\Kb_{mm}=k_\theta(\Ub,\Ub)$,  $\Kb_{rr}=k_\theta(\Xb_r,\Xb_r)$, and $\Kb_{mr}=\Kb_{rm}^{\top}=k_\theta(\Ub, \Xb_r)$. 
By introducing inducing points, the 
cost of learning the model is reduced from $\bigO(N^3)$
in $\log Z_{\phi, \theta}(\cdot)$ to $\bigO(Nm^2)$ in a
modified objective. 

We next describe two of the most
popular ways to learn the variational parameters $\Ub,\, \mub,\, \Ab$
that are based on a second \emph{inner} variational approximation
for the Gaussian process regression that lower bounds $\log Z_{\phi, \theta}(\cdot)$.
For this second inner variational inference, we aim to learn
a cheap $q_S(\cdot)$ (Equation~\ref{eq:svgp_approx_post}) that closely
approximates the expensive $q(\cdot)$ (Equation~\ref{eq:GP_reg_view}). 

\textbf{\citet{Titsias2009VariationalProcesses}}. %
Let $\zb=\zb_{1:N}^l$, then
the parameters $\Ub,\, \mub,\, \Ab$ may be learnt by
minimizing $\KL\big(q_S(\zb\,|\cdot) \: || \: q(\zb \,|\cdot) \big)$,
or equivalently by maximizing a lower bound to the 
marginal likelihood of the latent dataset $\log Z_{\phi, \theta}^{l}(\cdot)$.
Let $\bm{\Sigma} := \Kb_{mm} +\Kb_{mN}\diag(\sigbtt)\Kb_{Nm} \:$, 
then the optimal $\mub$ and $\Ab$ may be found analytically:
\begin{align} \label{eq:mu_star}
    \mub_T &= \Kb_{mm} \bm{\Sigma}^{-1} \Kb_{mN} \diag(\sigbtt) \ybt, \\
    \Ab_T &= \Kb_{mm}\bm{\Sigma}^{-1}\Kb_{mm},  \label{eq:A_star} %
\end{align}
where $\Kb_{mN}=k_\theta(\Ub, \Xb)$. Plugging $\mub_T$
and $\Ab_T$ back into the appropriate evidence lower
bound yields the final lower bound for learning $\Ub$ in the \textbf{T}itsias model
\begin{align}\label{eq:ELBO_T}
    \ELBO_{T}(\Ub, \phi, \theta)  &= \\
     \log \mathcal{N}\big(\ybt | &\textbf{0}, \: \Kb_{Nm} \Kb_{mm}^{-1}\Kb_{mN} + \diag(\sigbt^2)\big)  \nonumber\\
    \quad -& \frac{1}{2}Tr\big(\diag(\sigbtt)\:(\Kb_{NN} - \Kb_{Nm} \Kb_{mm}^{-1}\Kb_{mN})\big).\nonumber 
\end{align}
Note that the bound is a function of $\ybt$ and $\sigbt$
which depend on the inference network with parameters $\phi$
and the kernel matrices which depend upon $\theta$ hence we make
these arguments explicit.
In the full GP-VAE ELBO $\cL_P(\cdot)$, substituting $q_S(\cdot)$, $\cL_T(\cdot)$
in place of $q(\cdot)$, $\log Z_{\phi,\theta}(\cdot)$
yields a sparse GP-VAE ELBO that can be readily used to reduce
computational complexity of existing GP-VAE methods for a generic
dataset and an arbitrary GP kernel function.\footnote{As an aside, this sparse GP-VAE ELBO may also be derived in the standard way
using $\KL\big(q_S(\Zb|\cdot)||p_{\psi, \theta}(\Zb|\Yb, \Xb)\big)$, see Appendix B.4.} %

However, observe from 
Equations~\ref{eq:mu_star}, \ref{eq:A_star} and \ref{eq:ELBO_T} that 
the entire dataset $\{\Xb, \Yb\}$ enters through $\Kb_{NN}$
and $\ybt$, $\sigbt$ respectively. Therefore, this ELBO is not 
amenable to mini-batching and has large memory requirements.

\textbf{\citet{Hensman2013GaussianData}}. 
In order to make variational sparse GP regression amenable to mini-batching, \citet{Hensman2013GaussianData}
proposed an ELBO that lower bounds $\ELBO_{T}$ and, more importantly,  decomposes as a sum of terms over data points.
Adopting our notation with explicit parameters, the
\textbf{H}ensman ELBO is given by
\begin{multline} 
\label{eq:ELBO_H} 
    \ELBO_{H}(\Ub, \bm{\mu},  \textbf{A}, \phi, \theta) = - \KL\big(q_S(\fb_m|\cdot) \: || \: p_\theta(\fb_m|\cdot)\big) \\
    +\sum_{i=1}^N \bigg\{ \log\mathcal{N}\big(\tilde{y}_i | \bm{k}_i \Kb_{mm}^{-1}\bm{\mu}, \: \tilde{\sigma}_i^{-2} \big) 
       \: - \\ \frac{1}{2 \tilde{\sigma}_i^{2}} \: (\Tilde{k}_{ii} + Tr(\textbf{A}\: \Lambda_i))  \bigg\}.
\end{multline}

Above, $\bm{k}_i$ is the $i$-th row of $\Kb_{Nm}$,
$\Lambda_i = \Kb_{mm}^{-1}\bm{k}_i \bm{k}_i^\top\Kb_{mm}^{-1}$
and $\Tilde{k}_{ii}$ is the $i$-th diagonal element of the
matrix $\Kb_{NN} - \Kb_{Nm} \Kb_{mm}^{-1}\Kb_{mN}$. Due to
the decomposition over data points, the gradients $\nabla\ELBO_{H}(\cdot)$ in
stochastic or mini-batch gradient descent are unbiased and only the data in
the current batch are needed in memory for the gradient updates. Consequently, 
with batch size $b$ the GP complexity is further reduced to $\bigO(bm^2 + m^3)$. Note that for $\mub = \mub_T, \Ab= \Ab_T$ and $b =N$, $\ELBO_H(\cdot)$ recovers $\ELBO_T(\cdot)$ \citep{Hensman2013GaussianData}.

While this method may seem to meet our requirements, it has a fatal drawback. 
Firstly, it is not amortized as $\mub$ and $\Ab$ are not 
functions of the observed data $\{\Xb, \, \Yb\}$ but instead need to be optimized once
for each dataset.
Secondly, as a consequence, in the full GP-VAE ELBO $\cL_P(\cdot)$, 
substituting $q_S(\cdot)$, $\cL_H(\cdot)$ in place of $q(\cdot)$, $\log Z_{\phi, \theta}(\cdot)$
and simplifying yields the following expression
\begin{align}\label{eq:SVGP-VAE_Hensman_ELBO}
\cL_{PH}&(\Ub, \psi, \theta, \mub^{1:L}, \Ab^{1:L}) = \\
\sum_{i=1}^N \mathbb{E}_{q_S}  &\bigg[  \log \ppsi(\yb_i | \zb_i) \bigg] - \sum_{l=1}^L KL\big(q_S^{l}(\fb_m|\cdot) \: || \: p_\theta^l(\fb_m|\cdot)\big) \nonumber
\end{align}
where $q_S^l(\fb_m|\cdot) = \mathcal{N}(\fb_m|\mub^l, \Ab^l)$. 

Note that the ELBO above is not a function of the inference network
parameters $\phi$ (for the full derivation, we refer to Appendix B.1). %
The sparse approximate posterior is parameterized by $\Ub, \mub, \Ab, \theta$ which
are all treated as free parameters to be optimized, that is, they are not %
functions of the latent dataset or the inference network. 
Maximizing the full GP-VAE ELBO is equivalent to minimizing the KL divergence from
the approximate to the true posterior and neither of these depend upon the latent dataset or the inference network.
Therefore, using the Hensman sparse GP within an amortized GP-VAE model
causes the ELBO to be independent of the inference network parameters. 
Hence, this method also cannot be used as-is to amortize the sparse GP-VAE with mini-batches.

\subsection{The best of both ELBOs}
\label{sec:best_ELBOs}

Recall our goal to make GP-VAE models amenable to large datasets. This requires avoiding the large memory requirements and
being able to amortize inference. To alleviate these problems, \cite{Casale2018GaussianAutoencoders}
propose to use a Taylor approximation of the GP prior term in their ELBO.
However, this significantly increases implementation complexity and gives rise to
potential risks in ignoring curvature. 
We take a different approach utilising sparse GPs. We desire
a model that can scale to large datasets, like \citet{Hensman2013GaussianData},
while also being able to directly compute variational parameters from
the latent regression dataset, like \citet{Titsias2009VariationalProcesses}.
To this end, we take a mini-batch of the data, $\Xb_b \subset \Xb$, $\Yb_b\subset \Yb$,
and with the network $\tilde{q}_\phi(\cdot)$ create a mini-batch
of the latent dataset $\Xb_b$, $\ybt_b$, $\sigbt_b$. Following
\citet{Titsias2009VariationalProcesses}, with Equations~\ref{eq:mu_star}
and \ref{eq:A_star} for the optimal $\mub_T$ and $\Ab_T$, we
analytically compute stochastic estimates for each latent channel $l$ given by
\begin{align}\label{eq:MC_estimators}
    \bm{\Sigma}_{b}^l &:= \Kb_{mm} + \frac{N}{b} \Kb_{mb} \, \diag(\sigbtt_b) \, \Kb_{bm} , \nonumber \\
    \mub_{b}^l &:= \frac{N}{b} \Kb_{mm}\left({\bm{\Sigma}^l_b}\right)^{-1} \Kb_{mb}\, \diag(\sigbtt_b)\,\ybt^l_b, \nonumber\\
    \Ab_{b}^l &:= \Kb_{mm}\left({\bm{\Sigma}^l_b}\right)^{-1}\Kb_{mm}. \:
\end{align}
where $\Kb_{mb}=k_\theta(\Ub, \Xb_b)\in \R^{m\times b}$.
For a full derivation of these estimators, see Appendix B.2. %
All these estimators are consistent, so they converge to
the true values for $b \rightarrow N$.
However, while $\bm{\Sigma}^l_b$ is an unbiased
estimator for $\bm{\Sigma}^l$, the same does not
hold for $\mub^l_{b}$ and $\Ab_{b}^l$. We
investigate the magnitude of the bias in
Appendix C.4 %
finding that it is generally
small in practice. We believe this result to be in line
with sparse Gaussian process approximations that assume
the whole dataset may be summarized by a set of inducing
points. Alternatively, this may be interpreted as
assuming that the dataset contains redundancy,
that is, that we have more than enough data to learn the latent function. In such a
case, (cheaply) learning an average of latent functions of
multiple mini-batches would closely approximate (expensively)
learning one latent function using the full dataset.

$\mub_{b}^l$ and $\Ab_{b}^l$ parameterize the approximate posterior
$q_S(\cdot)$ which is, therefore, a direct function of the data $\Xb_b$, $\Yb_b$ 
and hence it is an amortized approximate posterior. By taking a mini-batch
of data, one may assume that we may also compute $\cL_T(\cdot)$
of the mini-batch latent dataset. However, note that such an $\cL_T(\cdot)$ 
is a lower bound for $\log Z_{\phi, \theta}(\cdot)$ of the mini-batch latent
dataset, not a lower bound for the full latent dataset.
Instead, we use $\mub_{b}^l$ and $\Ab_{b}^l$ along with $\Ub$ and $\theta$ 
to compute the GP evidence lower bound of
\citet{Hensman2013GaussianData} given in Equation \ref{eq:ELBO_H},
which is also suitable to mini-batching and lower bounds the
marginal likelihood of the full latent dataset. 
Finally, the evidence lower bound of our \textbf{S}parse (\textbf{V}ariational) \textbf{G}aussian \textbf{P}rocess
\textbf{V}ariational \textbf{A}uto\textbf{e}ncoder, for a single mini-batch $\Xb_b, \Yb_b$, is thus
\begin{multline}\label{eq:SVGP-VAE_ELBO}
    \ELBO_{SVGP-VAE}\big(\Ub, \psi, \phi, \theta):= \\
    \sum_{i=1}^b\E_{q_S}  \bigg[  \log \ppsi(\yb_i | \zb_i) - %
     \log \tilde{q}_\phi(\zb_i | \yb_i)\bigg] \\
    + \frac{b}{N}\sum_{l=1}^L \ELBO_{H}^{l}(\Ub, \, \phi, \, \theta, \, \mub_{b}^l, \,\, \Ab_{b}^l),
\end{multline}
where each $\cL^l_H(\cdot)$ is computed using the mini-batch
of the latent dataset $\Xb_b$, $\ybt_b^l$, $\tilde{\bm{\sigma}}_b^l$.
By naturally combining well known
approaches, we arrive at a sparse GP-VAE that is both amortized and
can be trained using mini-batches. The VAE parameters $\phi, \psi$, inducing points $\Ub$, and the GP kernel $\theta$ can all be optimized jointly in an end-to-end fashion as we show in the next section.

Also note that during training, $\mub_{b}^1,...,\mub^L_b$ and
$\Ab_{b}^1,...,\Ab_{b}^L$ are computed from a mini-batch
$\Xb_b$, $\Yb_b$. However at test time, given a new dataset, 
all available data $\Xb$, $\Yb$ may be used to compute the $\mub^1,..,\mub^L$ 
and $\Ab^1,...,\Ab^L$. The Gaussian process structure places
no theoretical restriction upon the number of observations
that are incorporated into the approximate posterior parameters,
any amount of data can be pooled simply according to the kernel operations.
In contrast, neural networks typically assume fixed input and
output sizes and pooling data in a principled way requires much more
attention.

While we have treated the auxiliary data $\Xb$ as observed throughout this section, our model can also be used when $\Xb$ is not given (or is only partly observed). In such cases, we make use of the Gaussian Process Latent Variable Model (GP-LVM) introduced by \cite{Lawrence2004GaussianData} to learn the missing part of $\Xb$, similar to what is done in \cite{Casale2018GaussianAutoencoders}. In SVGP-VAE, (missing parts of) $\Xb$ can be learned jointly with the rest of the model parameters.

\section{Experiments}
\label{sec:experiments}

We compared our proposed model with existing approaches measuring both performance and scalability on some simple synthetic data and large high-dimensional benchmark datasets. Implementation details can be found in Appendix A %
and additional experiments in Appendix C. %
The implementation of our model as well as our experiments are publicly available at \url{https://github.com/ratschlab/SVGP-VAE}.

\subsection{Synthetic moving ball data}
\label{sec:exp_ball}

\begin{figure}
    \centering
    \includegraphics[width=1.0\linewidth]{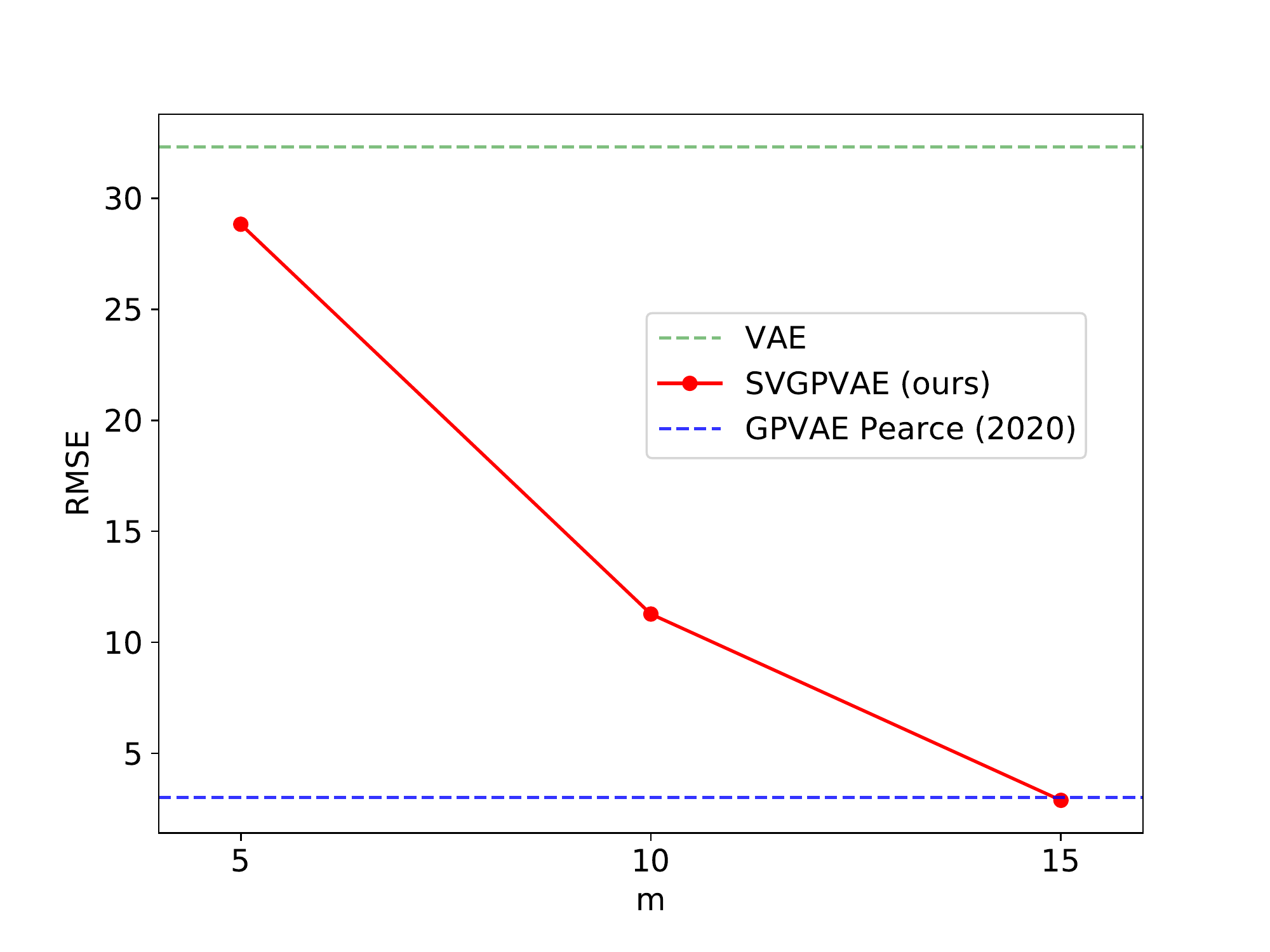}
    \caption{Performance of our SVGP-VAE models as a function of the number of inducing points. We see that as we increase the number of inducing points, the performance gracefully approaches the one of the exact GP-VAE baseline model.}
    \label{fig:ball_scaling_curve}
\end{figure}

The moving ball data was utilized in \citet{Pearce2019ThePixels}.
It consists of black-and-white videos of a moving circle, where the 2D trajectory is sampled from a GP with radial basis function (RBF) kernel.
The goal is to reconstruct the correct underlying trajectory in the two-dimensional latent space from the frames in pixel space.
Since the videos are short (30 frames), full GP inference is still feasible in this setting, such that we can compare our sparse approach against the gold standard. Note that due to the small dataset size we do not perform mini-batching within each video here.

\begin{figure}
    \centering
    \includegraphics[width=1.0\linewidth]{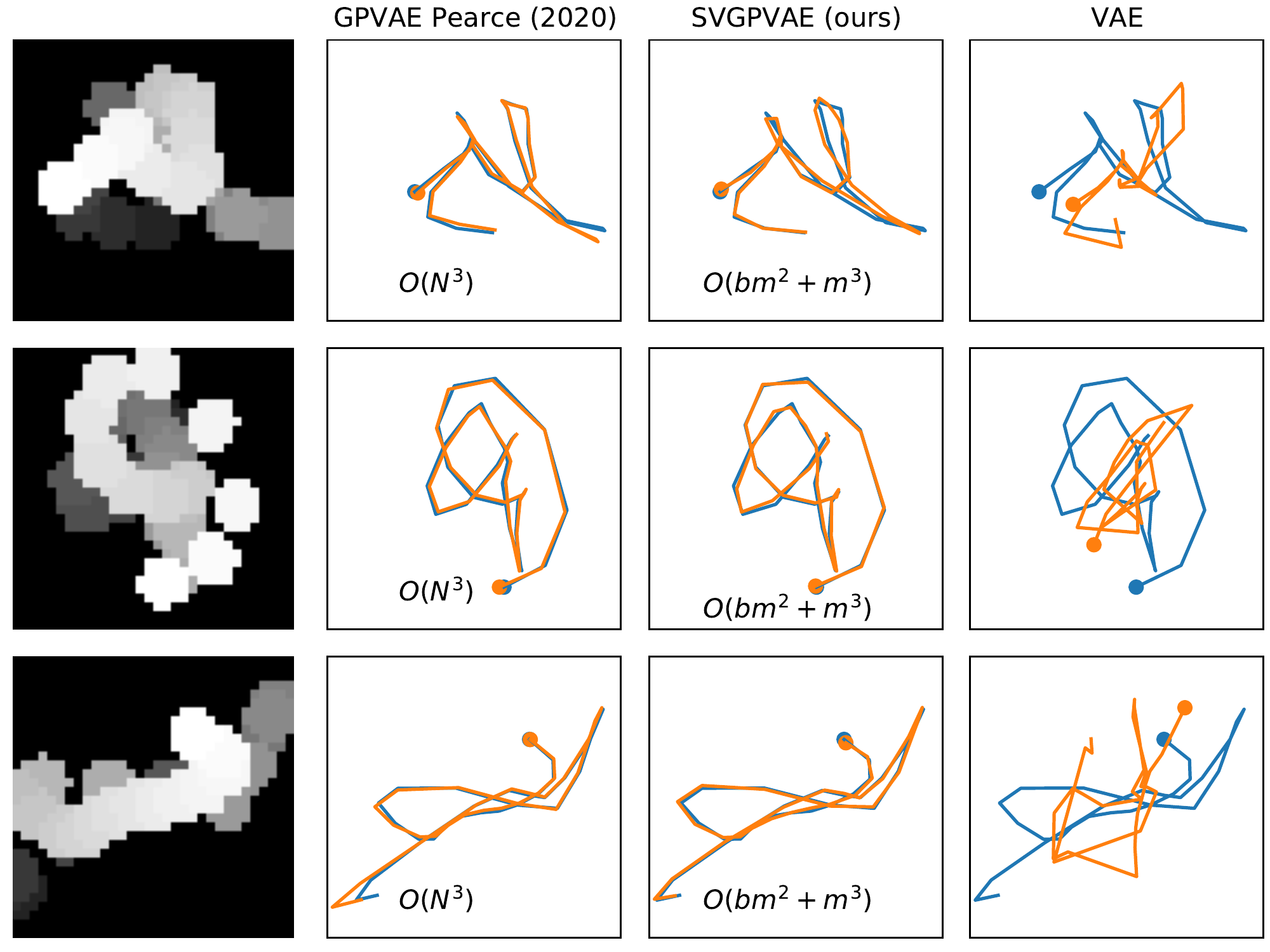}
    \caption{Reconstructions of the latent trajectories for the moving ball data. Frames of each test video are overlaid and shaded by time in the first column. Ground truth trajectories are depicted in blue, while predicted trajectories are shown in orange. We can see that the standard VAE fails to model the trajectories faithfully, while the GP-VAE models (including our sparse approximation) match them closely. Note that $b = N$ in SVGP-VAE for this experiment. For SVGP-VAE, the number of inducing points was set to $m=15$.}
    \label{fig:ball_trajectories}
\end{figure}

\paragraph{Scaling behavior.}
We see in Figure~\ref{fig:ball_scaling_curve} that as we increase the number of inducing points our method uses, its performance in terms of root mean squared error (RMSE) approaches the performance of the full GP baselines. It reaches the baseline performance already with 15 inducing points, which is half the number of data points in the trajectory and therefore  four times less computationally intensive than the baseline.
The reconstructions of the trajectories also qualitatively agree with the baseline, as can be seen in Figure~\ref{fig:ball_trajectories}.

\paragraph{Optimization of kernel parameters.}

Another advantage of our proposed method over the previous approaches is that it is agnostic to the kernel choice and even allows to optimize the kernel parameters (and thereby learn a better kernel) jointly during training. In \cite{Pearce2019ThePixels}, joint optimization of kernel parameters was not considered, while in \cite{Casale2018GaussianAutoencoders} a special training regime is deployed where VAE and GP parameters are optimized at different stages. Since the moving ball data is generated by a GP, we know the optimal kernel length scale for the RBF kernel in this case, which is namely the one of the generating process.
We optimized the length scale of our SVGP-VAE kernel and found that when using a sufficient number of inducing points, we indeed recover the true length scale almost perfectly (Fig.~\ref{fig:ball_lengthscales}).
Note that when too few inducing points are used, the \emph{effective} length scale of the observed process in the subspace spanned by these inducing points is indeed larger, since some of the variation in the data will be orthogonal to that subspace.
It is thus to be expected that our model would also choose a larger length scale to model the observations in this subspace.

\begin{figure}
    \centering
    \includegraphics[width=0.9\linewidth]{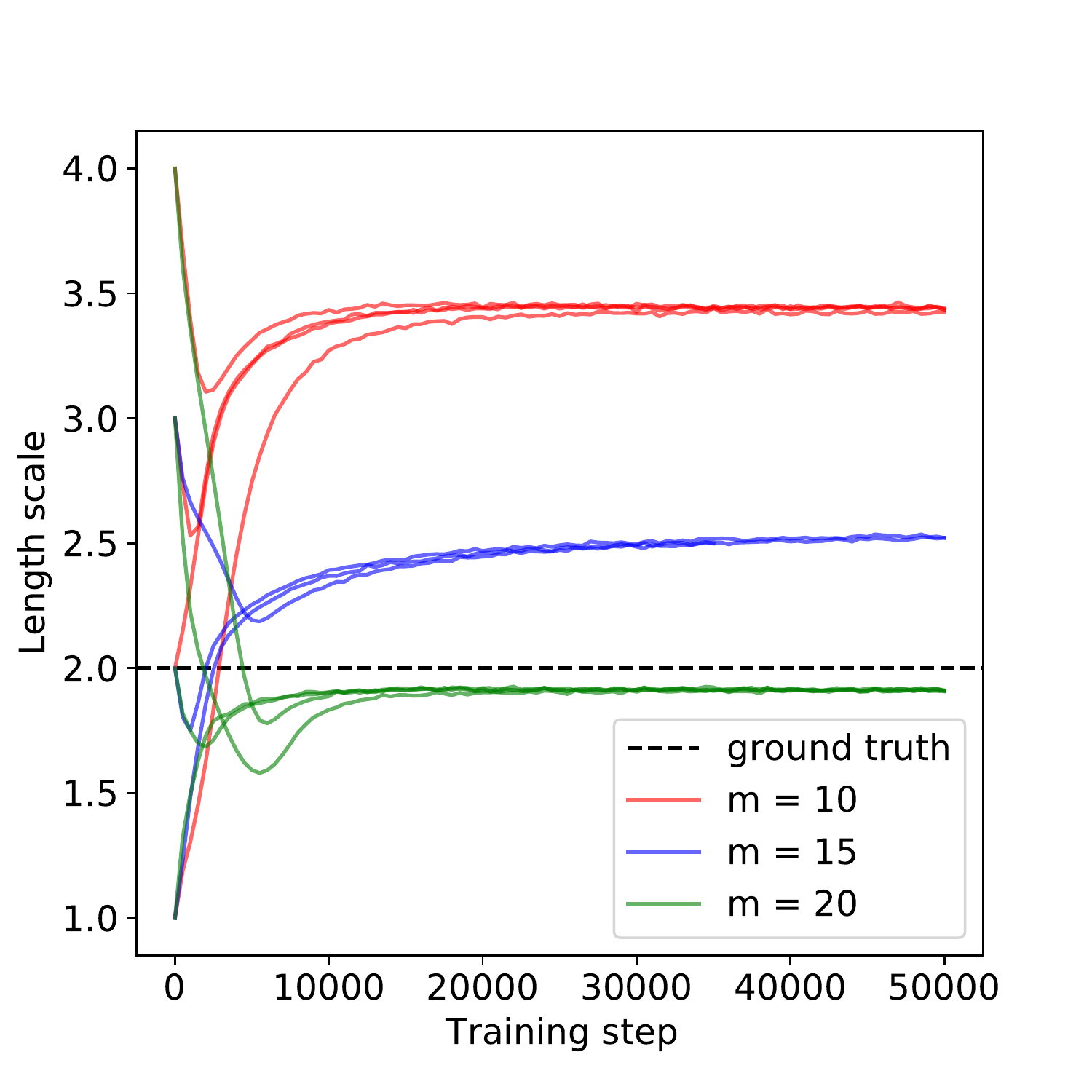}
    \caption{Optimized length scales of our SVGP-VAE model during training on the moving ball data. With sufficiently many inducing points, the model recovers the true length scale of the generating process.}
    \label{fig:ball_lengthscales}
\end{figure}

\paragraph{Optimization of inducing points.}
When working with sparse Gaussian processes, the selection of inducing point locations can often be crucial for the quality of the approximation \citep{Titsias2009VariationalProcesses, fortuin2018scalable, jahnichen2018scalable, burt2019rates}.
In our model, we can optimize these inducing point locations jointly with the other components.
On the moving ball data, since the trajectories are generated from stationary GPs, the optimal inducing point locations should be roughly equally spaced along the time dimension.
When we adversarially initialize the inducing points in a small region of the time series, we see that the model pushes them apart over the course of training and converges to this optimal spacing (Fig.~\ref{fig:ball_inducing_points}).
Together with the previous experiment, these observations suggest that the model is able to choose close-to-optimal inducing points and kernel functions in a data-driven way during the normal training process.

\subsection{Conditional generation of rotated MNIST digits}\label{sec:MNIST}

\begin{figure}
    \centering
    \includegraphics[width=1.0\linewidth]{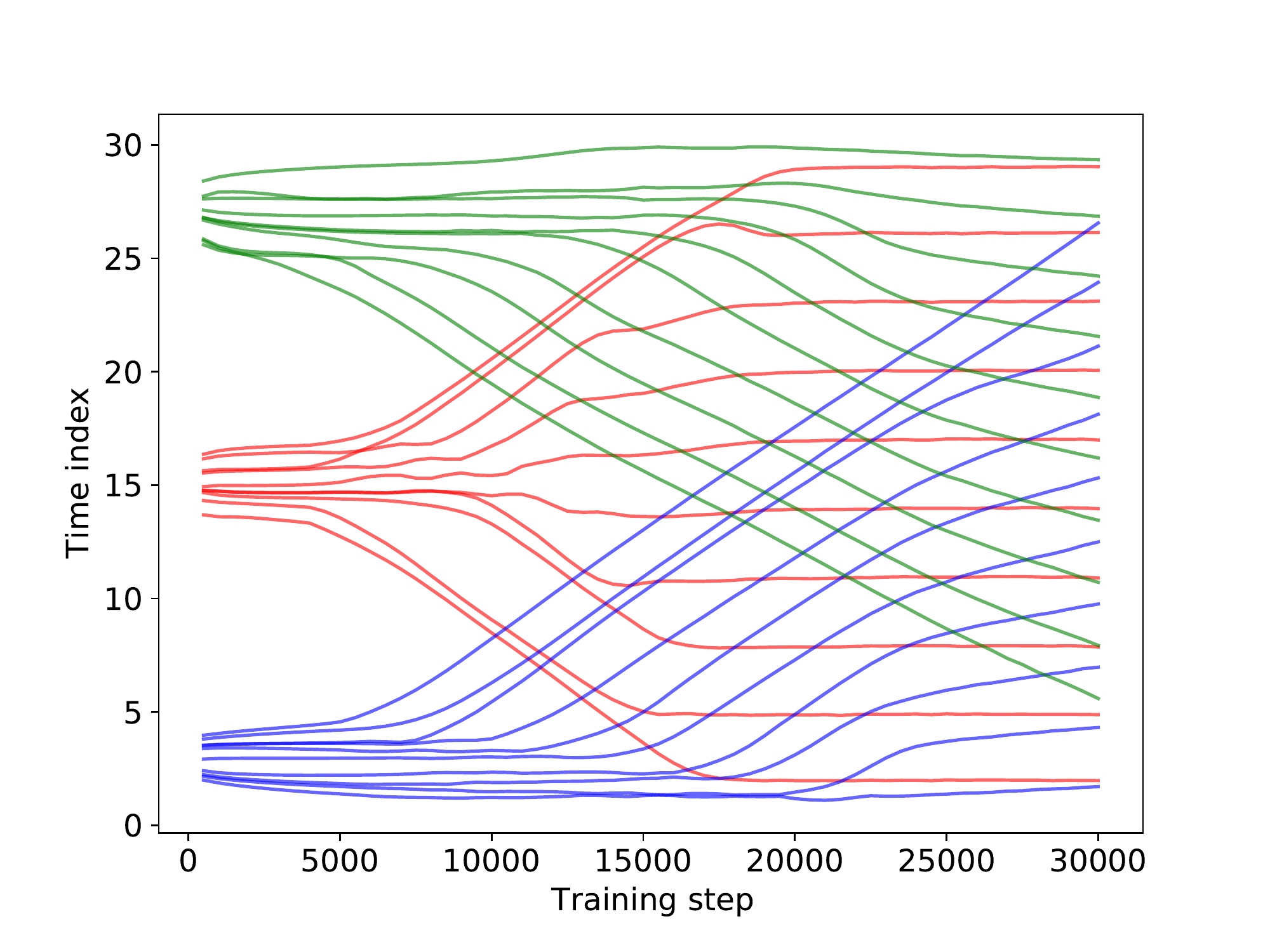}
    \caption{Optimized inducing points of our SVGP-VAE model during training on the moving ball data for three different (suboptimal) initializations. We can see that the model correctly learns to spread the inducing points evenly over the time series, which should be expected as a stationary GP kernel is used in the data generating process.}
    \label{fig:ball_inducing_points}
\end{figure}

To benchmark our model against existing scalable GP-VAE approaches, we follow the experimental setup from \citet{Casale2018GaussianAutoencoders} and use rotated MNIST digits \citep{lecun1998gradient} in a conditional generation task.
The task is to condition on a number of digits that have been rotated at different angles and to generate an image of one of these digits rotated at an unseen angle. In the original work, they consider 400 images of the digit 3, each rotated at multiple angles in $[0, 2\pi].$ Using identical architectures, kernel, and dataset ($N = 4050$), we report results
for both the GP-VAE of \citet{Casale2018GaussianAutoencoders} and our SVGP-VAE.
The full GP-VAE model from \citet{Pearce2019ThePixels} cannot be applied to this size of data, hence it is omitted.
As alternative baselines, we report results for a conditional VAE (CVAE) \citep{sohn2015learning} as well as for an extension of a sparse GP (SVIGP) approach from \cite{Hensman2013GaussianData}. We use the GECO algorithm \citep{Rezende2018TamingVAEs} to train our SVGP-VAE model, which greatly improves the stability of the training procedure.

\begin{table*}
\setlength{\tabcolsep}{10pt}
\centering
\caption{Results on the rotated MNIST digit 3 dataset. Reported here are mean values together with standard deviations based on 5 runs. We see that our proposed model performs comparably to the sparse GP baseline from \citet{Hensman2013GaussianData} and outperforms the VAE baselines while still being more scalable than the \citet{Casale2018GaussianAutoencoders} model.}
\begin{tabular}{l l l l}
\toprule
 & \textbf{MSE} & \textbf{GP complexity}   & \textbf{Time/epoch [s]}\\
\midrule
\textbf{CVAE} \citep{sohn2015learning} & $0.0796 \pm 0.0023$  & - & $0.39 \pm 0.01$ \\[0.08cm]
  \textbf{GPPVAE} \citep{Casale2018GaussianAutoencoders} & $0.0370 \pm 0.0012$ & $\bigO(NH^2)$ & $19.10 \pm 0.66$\\[0.08cm]
\textbf{SVGP-VAE} (ours) & $0.0251 \pm 0.0005
$ &  $\bigO(bm^2 + m^3)$ & $1.90 \pm 0.02$ \\[0.08cm]
\textbf{Deep SVIGP} \citep{Hensman2013GaussianData} & $0.0233 \pm 0.0014
$ &  $\bigO(bm^2 + m^3)$ & $1.15 \pm 0.04$ \\[0.08cm]
\bottomrule
\end{tabular}
\label{table:rot_MNIST_main}
\end{table*}

\paragraph{Performance of conditional generation.}
We see in Table~\ref{table:rot_MNIST_main} that our proposed model outperforms the VAE baselines in terms of MSE, while still being computationally more efficient than the model from \cite{Casale2018GaussianAutoencoders} (in theory and practice).\footnote{Note that in their paper, \citet{Casale2018GaussianAutoencoders} report a performance of 0.028 on this task. However, their code for the MNIST experiment is not openly available and we could not reproduce this result with our reimplementation (which is also available at \url{https://github.com/ratschlab/SVGP-VAE}).}
This can also be seen visually in Figure~\ref{fig:mnist_images} as our model produces the most faithful generations.
For the SVGP-VAE, the number of inducing points was set to $m=32$ and the batch size was set to $b = 256$. For the GP-VAE \citep{Casale2018GaussianAutoencoders}, the low-rank matrix factor $H$ depends on the dimension of the linear kernel $M$ used in their model ($M = 8$ and $H= 128$).

Moreover, our SVGP-VAE model comes close in performance to the unamortized sparse GP model with deep likelihood from \cite{Hensman2013GaussianData}.
This shows that the amortization gap of our model is small \citep{cremer2018inference}.
Note that this baseline was not considered in the previous GP-VAE literature \citep{Casale2018GaussianAutoencoders}, even though for the task of conditional generation, where we try to learn a single GP over the entire dataset, amortization is not strictly needed.
However, in tasks where the inference has to be amortized across several GPs, this model could not be used.
More details on this baseline are provided in Appendix \ref{sec:app_Hensman_baseline}.

\begin{figure}
    \centering
    \includegraphics[width=1.0\linewidth]{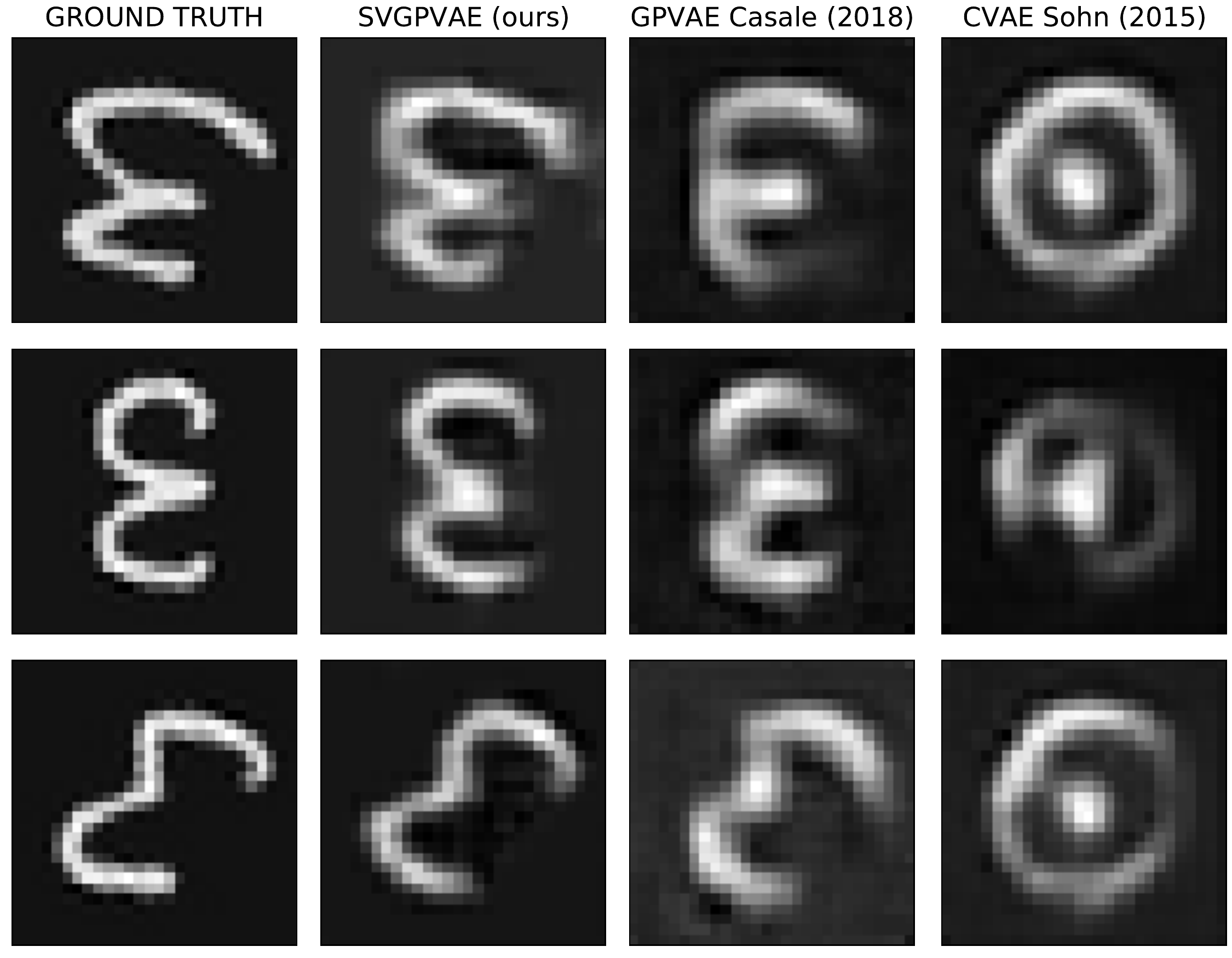}
    \caption{Conditionally generated rotated MNIST images. The generations of our proposed model are qualitatively more faithful to the ground truth. For more examples see Appendix C.3.}
    \label{fig:mnist_images}
\end{figure}

\begin{figure}
    \centering
    \includegraphics[width=1.0\linewidth]{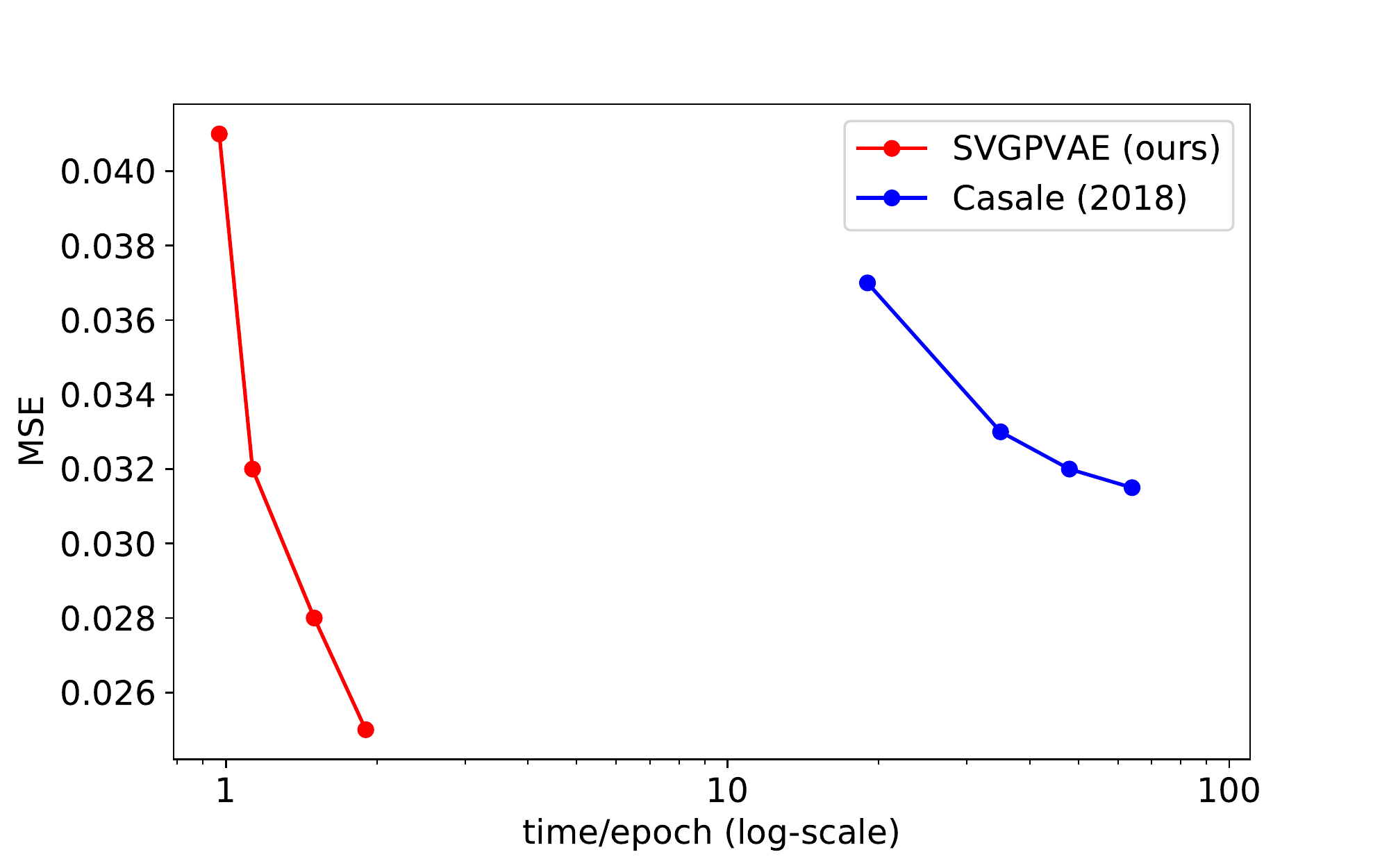}
    \caption{Performance of our proposed model with different numbers of inducing points and the \citet{Casale2018GaussianAutoencoders} model with different kernel dimensionalities as a function of runtime. For the SVGP-VAE, we consider four different configurations of inducing points, while for the \citet{Casale2018GaussianAutoencoders} model, we use four different dimensions of the linear kernel:  $m, M \in \{8, 16, 24, 32\}$.  }
    \label{fig:mnist_tradeoff}
\end{figure}

\paragraph{Tradeoff between runtime and performance.}
The performance of our sparse approximation can be increased by choosing a larger number of inducing points, at a quadratic cost in terms of runtime.
The \citet{Casale2018GaussianAutoencoders} model, while being more restricted in its kernel choice, offers a similar tradeoff between runtime and performance by choosing a different dimensionality for the low-rank linear kernel used in their latent space (see Appendix B.3). %
In Figure \ref{fig:mnist_tradeoff} we depict performance for both models when varying the number of inducing points and the dimension of the linear kernel, respectively. We observe that SVGP-VAE, besides being much faster, exhibits a steeper decline in the MSE as the model's capacity is increased.

\begin{figure}
    \centering
    \includegraphics[width=1.0\linewidth]{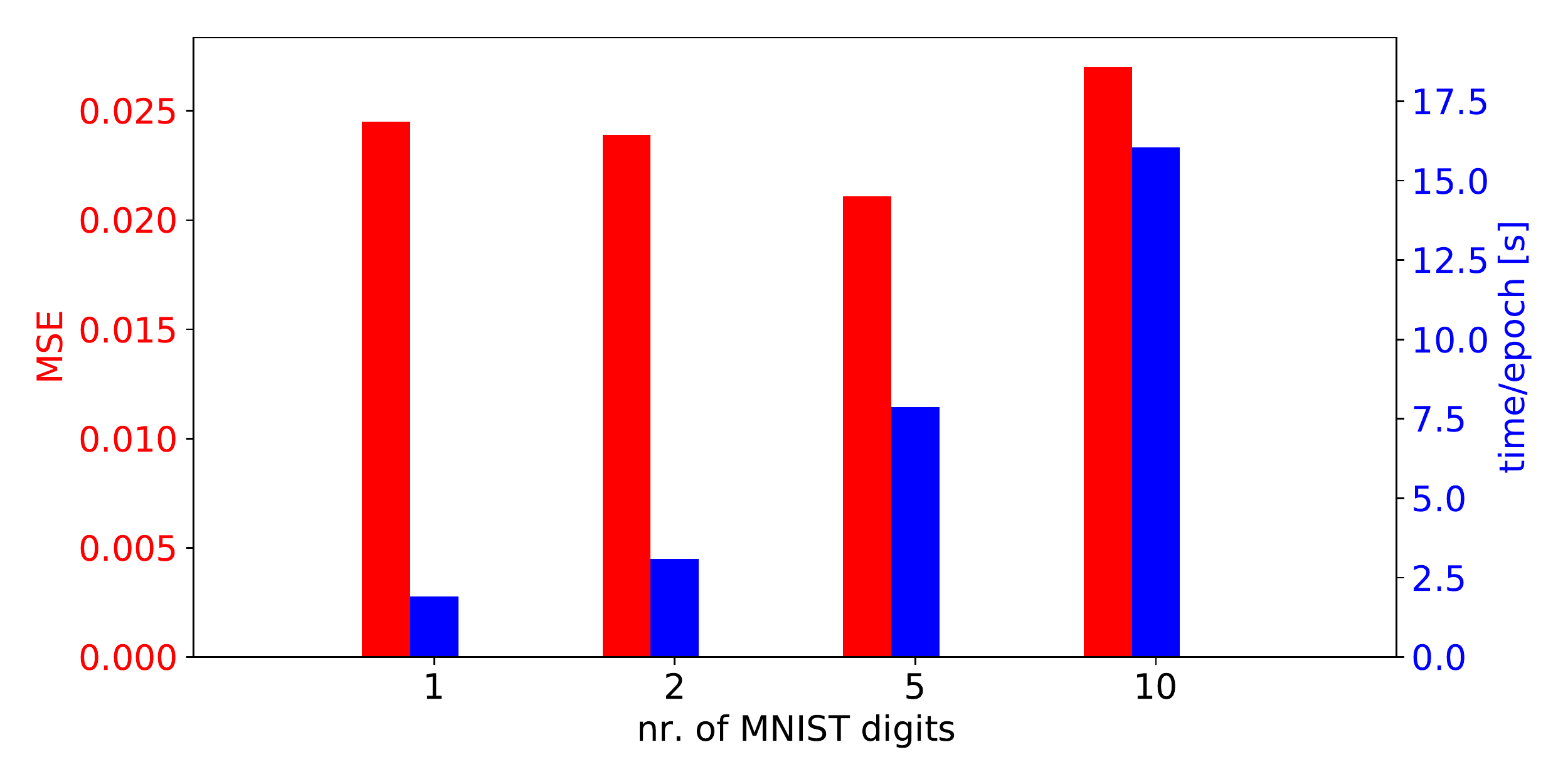}
    \caption{Performance and runtime of our proposed model on differently sized subsets of the MNIST dataset, including the full set. We see that the performance stays roughly the same, regardless of dataset size, while the runtime grows linearly as expected. The size of each dataset equals $4050 \times \textrm{nr. of MNIST digits}$.}
    \label{fig:mnist_scaling}
\end{figure}

\paragraph{Scaling to larger data.}
As mentioned above, \citet{Casale2018GaussianAutoencoders} restrict their experiment to a small subset of the MNIST dataset and indeed we did also not manage to scale their model to the whole dataset on our hardware (11~GB GPU memory).
Our SVGP-VAE, however, is easily scalable to such dataset sizes.
We report its performance on larger subsets of MNIST (including the full dataset) in Figure~\ref{fig:mnist_scaling}.
We see that the performance of our proposed model does not deteriorate with increased dataset size, while the runtime grows linearly as expected.
All in all, we thus see that our model is more flexible than the previous GP-VAE approaches, scales to larger datasets, and achieves a better performance at lower computational cost.

\subsection{SPRITES experiment}
\label{sec:SPRITES}
We additionally assessed the performance of our model on the SPRITES dataset \citep{li2018disentangled}. It consists of images of cartoon characters in different actions/poses. Each character has a unique style (skin color, tops, pants, hairstyle). There are in total 1296 characters, each observed in 72 different poses. For training, we use 1000 characters and we randomly sample 50 poses for each ($N = 50,000$). Auxiliary data for each image frame consists of a character style and a specific pose. The task is to conditionally generate characters not seen during training in different poses.

For the pose part of the auxiliary data, we use a GP-LVM \citep{Lawrence2004GaussianData}, similar to what was done in the rotated MNIST experiment for the digit style. Using the GP-LVM also for the character style would not allow us to extrapolate to new character styles during the test phase. To overcome this, we introduce a \textit{representation network}, with which we learn the unobserved parts of the auxiliary data in an amortized way.

Our model easily scales to the size of the SPRITES dataset (time per training epoch: $51.8 \pm 0.8$ seconds). Moreover, on the test set of 296 characters, our SVGP-VAE achieves a solid performance of $0.0079 \pm 0.0009$ pixel-wise MSE. In Figure \ref{fig:SPRITES}, we depict some generations for two test characters. We observe that model faithfully generates the pose information. However, it sometimes wrongly generates parts of the character style. We attribute this to the additional complexity of trying to amortize the learning of the auxiliary data. Extending our initial attempt of using the representation network for such purposes, together with more extensive benchmarking of our model performance, is left for future work. More details on the SPRITES experiment are provided in Appendix \ref{sec:app_SPRITES}.
\begin{figure}
    \centering
    \includegraphics[width=1.0\linewidth]{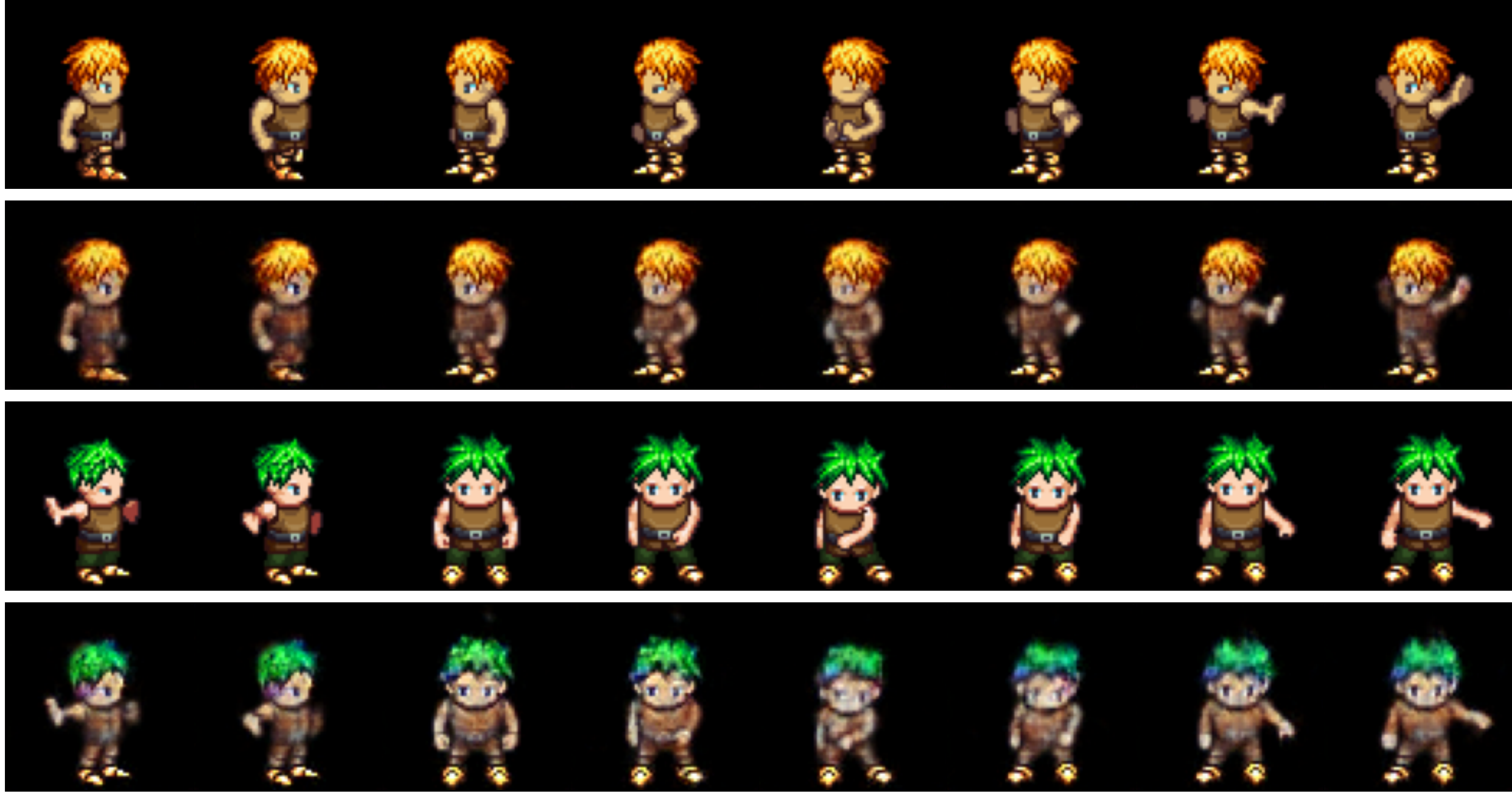}
    \caption{Conditionally generated SPRITES images for characters not observed during training. Images in the respective upper row are the ground truths, while the images in the respective lower row are conditional generations using our model.}
    \label{fig:SPRITES}
\end{figure}

\section{Conclusion}

We have proposed a novel sparse inference method for GP-VAE models and have shown theoretically and empirically that it is more scalable than existing approaches, while achieving competitive performance. Our approach bridges the gap between sparse variational GP approximations and GP-VAE models, thus enabling the utilization of a large body of work in the sparse GP literature. As such, it represents an important step towards unlocking the possibility to perform amortized GP regression on large datasets with complex likelihoods (e.g., natural images).

Fruitful avenues for future work include considering even more recently proposed sparse GP approaches \citep{Cheng2017VariationalComplexity, evans2020quadruply} and comparing our proposed scalable GP-VAE solution against other families of deep generative models \citep{ mirza2014conditional, Eslami1204}. This would help identify real-world applications where GP-VAEs could be most impactful.

\subsubsection*{Acknowledgements}
M.J. acknowledges funding from the Public Scholarship and Development Fund of the Republic of Slovenia.
V.F. was supported by a PhD fellowship from the Swiss Data Science Center and by the grant \#2017-110 of the Strategic Focus Area ``Personalized Health
and Related Technologies (PHRT)'' of the ETH Domain.
M.J. and V.F. were also supported by ETH core funding (to G.R.).
S.M. is supported by the Defense Advanced Research Projects Agency (DARPA) under Contract No. HR001120C0021. Any opinions, findings and conclusions or recommendations expressed in this material are those of the author(s) and do not necessarily reflect the views of the Defense Advanced Research Projects Agency (DARPA). Furthermore, S.M. was supported by the National Science Foundation under Grants 1928718, 2003237 and 2007719, and by Qualcomm.

\newpage

\bibliographystyle{abbrvnat}
\bibliography{references}

\onecolumn

\appendix
\counterwithin{figure}{section}
\counterwithin{table}{section}

\section{Experimental details}
\label{sec:exp-details}

Here we report the details and parameter settings for our experiments to foster reproducibility.

\subsection{Moving ball experiment}
For the moving ball experiment described in Section \ref{sec:exp_ball}, we use the same neural networks architectures and training setting as in \cite{Pearce2019ThePixels}.

\begin{table}[h!]
    \centering
    \caption{Parameter settings for the moving ball experiment.}
    \begin{tabular}{lc}
    \toprule
        Parameter & Value \\
        \midrule
         Nr. of feedforward layers in inference network & 2  \\
         Nr. of feedforward layers in generative network & 2 \\
         Width of a hidden feedforward layer & 500 \\
         Dimensionality of latent space (L) & 2\\
         Activation function & \emph{tanh} \\
         Learning rate & 0.001 \\
         Optimizer & Adam \\
         Nr. of epochs & 25000 \\
         Nr. of frames in each video (N) & 30 \\
         Dimension of each frame & $32 \times 32$ \\
         \bottomrule
    \end{tabular}
    \label{tab:ball_parameters}
\end{table}
A squared-exponential GP kernel with length scale $l=2$ was used. For the exact data generation procedure, we refer to \cite{Pearce2019ThePixels}. During training, 35 videos were generated in each epoch. The test MSE is reported on a held-out set of 350 videos. For the Adam optimizer \citep{Kingma2014Adam:Optimization}, the default Tensforflow parameters are used.

\subsection{MNIST experiment}
\label{sec:app_rot_MNIST}

For the rotated MNIST experiment described in Section \ref{sec:MNIST}, we used the same neural networks architectures as in \cite{Casale2018GaussianAutoencoders}: three convolutional layers followed by a fully connected layer in the inference network  and vice-versa in the generative network.

\begin{table}[h!]
    \centering
    \caption{Neural networks architectures for the MNIST experiment.}
    \begin{tabular}{lc}
    \toprule
        Parameter & Value \\
        \midrule
         Nr. of CNN layers in inference network &  3 \\
         Nr. of CNN layers in generative network &  3 \\
         Nr. of filters per CNN layer & 8 \\
         Filter size & $3 \times 3$ \\
         Nr. of feedforward layers in inference network & 1 \\
         Nr. of feedforward layers in generative network & 1 \\
         Activation function in CNN layers & ELU \\
         Dimensionality of latent space (L) & 16 \\
         \bottomrule
    \end{tabular}
    \label{tab:mnist_parameters}
\end{table}
The SVGP-VAE model is trained for 1000 epochs with a batch size of 256. The Adam optimizer \citep{Kingma2014Adam:Optimization} is used with its default parameters and a learning rate of 0.001. Moreover, the GECO algorithm \citep{Rezende2018TamingVAEs} was used for training our SVGP-VAE model in this experiment. The reconstruction parameter in GECO was set to $\kappa = 0.020$ in all reported experiments. 

For the GP-VAE model from \cite{Casale2018GaussianAutoencoders}, we used the same training procedure as reported in \cite{Casale2018GaussianAutoencoders}. We have observed in our reimplementation that a joint optimization at the end does not improve performance. Hence, we report results for the regime where the VAE parameters are optimized for the first 100 epochs, followed by 100 epochs during which the GP parameters are optimized. Moreover, we could not get their proposed low-memory modified forward pass to work, so in our reimplementation the entire dataset is loaded into the memory at one point during the forward pass. Our reimplementation of the GP-VAE model from \cite{Casale2018GaussianAutoencoders} is publicly available at \url{https://github.com/ratschlab/SVGP-VAE}.

For both models, the GP kernel proposed in \cite{Casale2018GaussianAutoencoders} is used. For more details on the kernel, we refer to Appendix \ref{sec:app_low_rank}. Note that the auxiliary data $\Xb$ is only partially observed in this experiment --- for both models we use a GP-LVM to learn the missing parts of $\Xb$. For both models, we use Principal Component Analysis (PCA) to initialize the GP-LVM vectors, as it was observed to lead to a slight increase in performance. PCA is also used in SVGP-VAE to initialize the inducing points. For more details see Appendix \ref{app:PCA_init}.

\subsection{SPRITES experiment}
\label{sec:app_SPRITES}

For the SPRITES experiment described in Section \ref{sec:SPRITES}, we used similar neural networks architectures as for the rotated MNIST experiment. Details are provided in Table \ref{tab:SPRITES_parameters}.

\begin{table}[h!]
    \centering
    \caption{Neural networks architectures for the SPRITES experiment.}
    \begin{tabular}{lc}
    \toprule
        Parameter & Value \\
        \midrule
         Nr. of CNN layers in inference network &  6 \\
         Nr. of CNN layers in generative network &  6 \\
         Nr. of filters per CNN layer & 16 \\
         Filter size & $3 \times 3$ \\
         Nr. of feedforward layers in inference network & 1 \\
         Nr. of feedforward layers in generative network & 1 \\
         Activation function in CNN layers & ELU \\
         Dimensionality of latent space (L) & 64 \\
         \bottomrule
    \end{tabular}
    \label{tab:SPRITES_parameters}
\end{table}
The SVGP-VAE model is trained for 50 epochs with a batch size of 500. The Adam optimizer \citep{Kingma2014Adam:Optimization} is used with its default parameters and a learning rate of 0.001. Moreover, the GECO algorithm \citep{Rezende2018TamingVAEs} was used for training our SVGP-VAE model in this experiment. The reconstruction parameter in GECO was set to $\kappa = 0.0075$.

The auxiliary data $\Xb$ is fully unobserved in this experiment. Recall that in SPRITES, the auxiliary data has two parts $\Xb = [\Xb_{s}, \: \Xb_{a}]$, with $\Xb_s \in \R^{N \times p_1}$ containing information about the character \textbf{s}tyle and $\Xb_a \in \R^{N \times p_2}$ containing information about the specific \textbf{a}ction/pose. Let $\xb_i = [\xb_{s,i}  \: \xb_{a,i}]$ denote auxiliary data for the $i$-th image (corresponding to the $i$-th row of the $\Xb$ matrix). A product kernel between two linear kernels is used:\footnote{$\delta_{ij} = 1$ if $i = j$ and $0$ else.}
\begin{align*}
    k_{\theta}(\xb_i, \xb_j) = \frac{\xb_{s,i}^T \xb_{s,j}}{\norm{\xb_{s,i}}\norm{\xb_{s,j}}} \cdot \frac{\xb_{a,i}^T \xb_{a,j}}{\norm{\xb_{a,i}}\norm{\xb_{a,j}}} + \sigma^2 \cdot \delta_{ij}  \; .
\end{align*}
The kernel normalization and the addition of the diagonal noise are used to improve the numerical stability of kernel matrices.

To learn the action part of the auxiliary data $\Xb_a$, we rely on a GP-LVM \citep{Lawrence2004GaussianData}, that is, we try to directly learn the matrix $\bm{A} \in \R^{72 \times p_2}$ consisting of GP-LVM vectors that each represent a specific action/pose. Since we want to extrapolate to new characters during the test phase\footnote{Note that an easier version of the SPRITES experiment would be to generate actions for characters already seen during the training phase. Such a conditional generation task would closely resemble the one from the face experiment in \cite{Casale2018GaussianAutoencoders}.}, the GP-LVM approach can not be used to learn the part of the auxiliary data that captures the character style information $\Xb_s$. This would require rerunning the optimization at test time to obtain a corresponding GP-LVM vector for the new, previously unseen style. To get around this, we introduce the \emph{representation network} $r_{\zeta}: \R^K \to \R^{p_1}$, similar to what is done in \cite{AliEslami2018NeuralRendering}, with which we aim to amortize the learning of the unobserved parts of the auxiliary data. Specifically, the representation for the $i$-th character style is then
\begin{align*}
    \mathbf{s}_i = f\big(r_{\zeta}(\yb_{1}), \: \dots, \: r_{\zeta}(\yb_{N_i})\big) \in \R^{p_1} \: , 
\end{align*}
where $\Yb_i = [\yb_1 \dots \yb_{N_i}]^T \in \R^{N_i \times K}$ represents all images of the $i$-th character, and $f$ is a chosen aggregation function (in our experiment we used the sum function). Instead of the GP-LVM vectors, the parameters of the representation network $\zeta$ are jointly optimized with the rest of the SVGP-VAE parameters. During training, we pass all 50 images (50 different actions) for each character through $r_{\zeta}$ to obtain the corresponding style representation. During the test phase, we first pass 36 actions through $r_{\zeta}$ and then use the resulting style representation vector to conditionally generate the remaining 36 actions. To help with the stability of training, we additionally pretrain the representation network on the classification task using the training data. Concretely, we train a classifier on top of the representations of the training data $r_{\zeta}(\yb_i), \: i=1, ..., N$. The (pretraining) label for each representation is a given character ID.\footnote{Recall that there are 1000 different characters in our training dataset, i.e., the pretraining task is a 1000-class classification problem.} 

The details on the architecture of the representation network are provided in Table \ref{tab:SPRITES_reprnn} (it is essentially a downsized inference network). 

\begin{table}[h!]
    \centering
    \caption{The architecture for the representation network $r_{\zeta}$ and some additional parameters in the SPRITES experiment.}
    \begin{tabular}{lc}
    \toprule
        Parameter & Value \\
        \midrule
         Nr. of CNN layers &  3 \\
         Nr. of filters per CNN layer & 16 \\
         Filter size & $2 \times 2$ \\
         Nr. of pooling layers & 1 \\
         Activation function in CNN layers & ELU \\
         Dimensionality of style representation ($p_1$) & 16 \\
         Dimensionality of action GP-LVM vectors ($p_2$) & 8
         \\
         Nr. of epochs for pretraining of $r_{\zeta}$ & 400 \\
         \bottomrule
    \end{tabular}
    \label{tab:SPRITES_reprnn}
\end{table}

\subsection{On the training of GP-VAE models (a practitioner's perspective)}
While working on implementations of different GP-VAE models, we have noticed that balancing the absolute magnitudes of the reconstruction and the KL-term is critical for achieving optimal results, even more so than in standard VAE models. In \cite{Fortuin2019GP-VAE:Imputation}, this was tackled by introducing a weighting $\beta$ parameter, whereas in \cite{Casale2018GaussianAutoencoders} a CV search on the noise parameter $\sigma_y^2$ of the likelihood $p_{\psi}(\yb_i | \zb_i)$ is performed. One downside of both solutions is that they introduce (yet) another  training  hyperparameter that needs to be manually tuned for every new dataset/model architecture considered. 

To get around this, we instead used the GECO algorithm \citep{Rezende2018TamingVAEs} to train our SVGP-VAE in the rotated MNIST experiment. Compared to the original GECO algorithm in \cite{Rezende2018TamingVAEs}, where the maximization objective is the KL divergence between a standard Gaussian prior and the variational distribution, the GECO maximization objective in the SVGP-VAE is composed of a cross-entropy term $\mathbb{E}_{q_S}[\log\Tilde{q}_{\phi}(\cdot)]$ and a sparse GP ELBO $\ELBO_{H}(\cdot)$. We have observed that GECO greatly simplifies training of GP-VAE models as it eliminates the need to manually tune the different magnitudes of the ELBO terms. Based on this, we would make a general recommendation for GECO to be used for training such models.

\newpage

\section{Supporting derivations}
\label{sec:app_supp_deriv}

\subsection{Vanishing of the inference network in GP-VAE with ELBO from \cite{Hensman2013GaussianData}}
\label{sec:app_vanishing_SVGP-VAE_Hensman}

In this section we show that working with the sparse GP approach presented in \cite{Hensman2013GaussianData} leads to vanishing of the inference network parameters $\phi$ in the GP-VAE model from \cite{Pearce2019ThePixels}. Recall that the sparse GP posterior from \cite{Hensman2013GaussianData} for the $l$-th latent channel has the form
\begin{align*}
    q_S^{l} (\zb^l_{1:N} | \cdot) =  \mathcal{N}\big(\zb^l_{1:N}| \Kb_{Nm} \Kb_{mm}^{-1} \bm{\mu}^l, \: \Kb_{NN} - \Kb_{Nm}\Kb_{mm}^{-1}\Kb_{mN} + \Kb_{Nm}\Kb_{mm}^{-1}\textbf{A}^l \Kb_{mm}^{-1} \Kb_{mN}\big) \: 
\end{align*}
with $\mub^l \in \R^m, \Ab^l \in \R^{m \times m}$ as free variational parameters, while the sparse GP ELBO for the $l$-th latent channel is given as 
\begin{multline*}
    \ELBO_{H}^{l}(\Ub, \bm{\mu}^l, \textbf{A}^l, \phi, \theta) = 
    \sum_{i=1}^N \bigg\{\log\mathcal{N}\big(\Tilde{y}_{l,i} \: | \: \bm{k}_i^T \Kb_{mm}^{-1}\bm{\mu}^l, \: \Tilde{\sigma}^{-2}_{l,i}\big) - \frac{1}{2 \Tilde{\sigma}^{-2}_{l,i}} \big( \Tilde{k}_{ii} + Tr(\textbf{A}^l\: \Lambda_i)\big) \bigg\} \\ - KL\big(q^{l}_S(\fb_m | \cdot) \: || \: p_{\theta}(\fb_m | \cdot)\big)
\end{multline*}
with $q^{l}_S(\fb_m | \cdot) = \mathcal{N}(\fb_m | \bm{\mu^l}, \bm{A^l})$ and $p_{\theta}(\fb_m | \cdot) = \mathcal{N}(\fb_m | \bm{0}, \: \Kb_{mm})$. 
 $\bm{k}_i$ represents the $i$-th column of $\Kb_{mN}$, $\Lambda_i := \Kb_{mm}^{-1}\bm{k}_i \bm{k}_i^T\Kb_{mm}^{-1}$ and $\Tilde{k}_{ii}$ is the $i$-th diagonal element of $\Kb_{NN} - \Kb_{Nm} \Kb_{mm}^{-1}\Kb_{mN}$. As mentioned in Section \ref{sec:methods}, $\ELBO_H^{l}$ depends on the inference network parameters $\phi$ through the (amortized) $l$-th latent dataset $\ybt_l=\mu^l_\phi(\Yb), \:\sigbt_l = \sigma^l_\phi(\Yb)$.
 
Note that the full sparse GP posterior equals $q_S(\Zb) = \prod_{l=1}^L q_S^{l} (\zb^l_{1:N} | \cdot) $. Similarly, the full sparse GP ELBO is $\ELBO_H = \sum_{l=1}^L \ELBO_{H}^{l}(\Ub, \bm{\mu}^l, \textbf{A}^l, \phi, \theta)$.

\vspace{10pt}
\begin{proposition}
For the $l$-th latent channel in the GP-VAE model with the bound from \cite{Hensman2013GaussianData}, the following relation holds:
\begin{align*}
    \mathbb{E}_{q_S^l}\big[\log\Tilde{q}_{\phi}(\zb^{\,l}_{1:N} | \Yb)\big] =
    \sum_{i=1}^N \bigg\{\log\mathcal{N}\big(\Tilde{y}_{l,i} \: | \: \bm{k}_i^T \Kb_{mm}^{-1}\bm{\mu}^l, \: \Tilde{\sigma}^{-2}_{l,i}\big) - \frac{1}{2 \Tilde{\sigma}^{-2}_{l,i}} \big( \Tilde{k}_{ii} + Tr(\textbf{A}^l\: \Lambda_i)\big) \bigg\} \: .
\end{align*}
\end{proposition}

\begin{proof}\label{prop:hensman_vanish_1}
For notational convenience, define $\Tilde{D}_l := \diag(\sigbt_l^2)$ and $\bm{B} := \Kb_{Nm}\Kb_{mm}^{-1}$. Also recall that $ \Tilde{q}_{\phi}(\zb^{l}_{1:N} | \Yb) = \mathcal{N}\big( \zb^{l}_{1:N} | \ybt_l, \: \Tilde{D}_l\big)$. Using the formula for the cross-entropy between two multivariate Gaussian distributions, we proceed as
\begin{align*}
    \mathbb{E}_{q_S^l}\big[\log\Tilde{q}_{\phi}(\zb^{l}_{1:N} | \Yb)\big] &= -\frac{N}{2}\log (2\pi) -\frac{1}{2} \log |\Tilde{D}_l| 
    - \frac{1}{2}\big(\ybt_l - \bm{B}\bm{\mu}^l\big)^T \Tilde{D}_l^{-1}\big(\ybt_l - \bm{B}\bm{\mu}^l\big) - \frac{1}{2} Tr\big(\Tilde{D}_l^{-1}(\Tilde{\Kb} + \bm{B}\bm{A}^l\bm{B}^T)\big) \\
    &= \log \mathcal{N}\big(\ybt_l | \bm{B} \bm{\mu^l}, \: \Tilde{D}_l \big) - \frac{1}{2} Tr(\Tilde{D}_l^{-1}\Tilde{\Kb}) - \frac{1}{2}Tr(\Tilde{D}_l^{-1}\bm{B}\bm{A}^l\bm{B}^T) \: .
\end{align*}
It remains to show that the last trace term equals $\sum_{i=1}^N \Tilde{\sigma}^{-2}_{l, i}Tr(\textbf{A}^l\: \Lambda_i)$, which follows from
\begin{align*}
    Tr(\Tilde{D}_l^{-1}\bm{B}\bm{A}^l\bm{B}^T) = Tr(\bm{A}^l\bm{B}^T\Tilde{D}_l^{-1}\bm{B}) =  Tr\big(\bm{A}^l\Kb_{mm}^{-1}\big(\sum_{i=1}^N \Tilde{\sigma}^{-2}_{l, i}\bm{k_i} \bm{k_i}^T\big)\Kb_{mm}^{-1}\big) = \sum_{i=1}^N \Tilde{\sigma}^{-2}_{l, i} Tr(\bm{A}^l\Lambda_i) \: .
\end{align*}
\end{proof}

\begin{proposition}
The GP-VAE ELBO with the bound from \cite{Hensman2013GaussianData} reduces to 
\begin{align*}
    \cL_{PH}(\Ub, \psi, \theta, \mub^{1:L}, \Ab^{1:L}) =
\sum_{i=1}^N \mathbb{E}_{q_S}  \bigg[  \log \ppsi(\yb_i | \zb_i) \bigg] - \sum_{l=1}^L KL\big(q_S^{l}(\fb_m|\cdot) \: || \: p_\theta^l(\fb_m|\cdot)\big)
\end{align*}
\end{proposition}

\begin{proof}
Using the above proposition, we have 
\begin{align*}
    & \mathbb{E}_{q_S}  \bigg[ \sum_{i=1}^N \log p_{\psi}(\yb_i | \zb_i) - \log \Tilde{q}_{\phi}(\zb_i | \yb_i)\bigg] + \sum_{l=1}^L \ELBO_{H}^{l}  \\
    &= \mathbb{E}_{q_S}  \bigg[ \sum_{i=1}^N \log p_{\psi}(\yb_i | \zb_i)\bigg] - \mathbb{E}_{q_S}  \bigg[ \sum_{l=1}^L\log \Tilde{q}_{\phi}(\zb^{l}_{1:N} | \Yb)\bigg] + \sum_{l=1}^L \ELBO_{H}^{l}  \\
    &= \mathbb{E}_{q_S}  \bigg[ \sum_{i=1}^N \log p_{\psi}(\yb_i | \zb_i)\bigg] - \sum_{l=1}^L\ \bigg(\mathbb{E}_{q_{S}^l}  \big[ \log \Tilde{q}_{\phi}(\zb^{l}_{1:N} | \Yb)\big] -  \ELBO_{H}^{l}\bigg)  \\
    &= \mathbb{E}_{q_S}  \bigg[ \sum_{i=1}^N \log p_{\psi}(\yb_i | \zb_i) \bigg] - \sum_{l=1}^L KL\big(q^{l}_S(\fb_m | \cdot) \: || \: p_{\theta}(\fb_m | \cdot)\big) \: .
\end{align*}
\end{proof}
Observe that in $\ELBO_{PH}(\cdot)$ all terms that include $\ybt_l$ or $\sigbt_l$ cancel out, hence such ELBO is independent of the inference network parameters $\phi$.

\subsection{Monte Carlo estimators in the SVGP-VAE}
\label{sec:app_MC_SVGP-VAE}

The idea behind the estimators used in $q_S$ in our SVGP-VAE is based on the work presented in \cite{evans2020quadruply}. The main insight is to rewrite the matrix operations as expectations with respect to the empirical distribution of the training data. Those expectations are then approximated with Monte Carlo estimators.

 Recall that the (amortized) latent dataset for the $l$-th channel is denoted by $\{\Xb, \ybt_l, \sigbt_l\}$, with $\ybt_l:=\mu^l_\phi(\Yb)$ and $\sigbt_l := \sigma^l_\phi(\Yb)$. For notational convenience, additionally denote $\Tilde{D}_l := \diag(\sigbt_l^2)$. First, observe that the matrix product $\Kb_{mN}\Tilde{D}_l^{-1}\Kb_{Nm}$ in $\bm{\Sigma}^l$ can be rewritten as a sum over data points $\sum_{i=1}^N B_i (\xb_i, \yb_i)$
with
\begin{align*}
        B_i(\xb_i, \yb_i) := \frac{1}{\Tilde{\sigma}^2_{l, i}} \begin{bmatrix}
k_{\theta}(\ub_1, \xb_i)k_{\theta}(\ub_1, \xb_i)  & \hdots & k_{\theta}(\ub_1, \xb_i)k_{\theta}(\ub_m, \xb_i) \\[0.6em]
\vdots & \ddots & \vdots\\[0.6em]
k_{\theta}(\ub_m, \xb_i)k_{\theta}(\ub_1, \xb_i) & \hdots & k_{\theta}(\ub_m, \xb_i)k_{\theta}(\ub_m, \xb_i) 
\end{bmatrix} \: .
\end{align*}
Let $\Bar{b}$ represent a set of indices of data points in the current batch with size $b$.
Moreover, define $\Kb_{bm} \in \R^{b \times m}, \: \Tilde{D}_{l, b} \in \R^{b \times b}, \: \ybt_b^l \in \R^b$ as the sub-sampled versions of $\Kb_{Nm} \in \R^{N \times m}, \: \Tilde{D}_{l} \in \R^{N \times N}$ and $\ybt_l \in \R^{N}$, respectively, consisting only of data points in $\bar{b}$. An (unbiased) Monte Carlo estimator for $\bm{\Sigma}^l$ is then derived as follows
\begin{align*}
    \bm{\Sigma}^l &= \Kb_{mm} + \Kb_{mN}\Tilde{D}_l^{-1}\Kb_{Nm} = 
    \Kb_{mm} + N\sum_{i=1}^N \frac{1}{N}B_i (\xb_i, \yb_i) = 
    \Kb_{mm} + N \cdot \mathbb{E}_{i \sim \{1, ..., N \} } \big[B_i(\xb_i, \yb_i)\big] \\ & \approx 
    \Kb_{mm} + \frac{N}{b}\sum_{i \in \Bar{b}} B_i(\xb_i, \yb_i) =
    \Kb_{mm} + \frac{N}{b}\Kb_{mb} \Tilde{D}_{l, b}^{-1} \Kb_{bm}=: \bm{\Sigma}^l_b \: .
\end{align*}
Additionally, define $\bm{c}_l :=  \Kb_{mN} \Tilde{D}_l^{-1} \ybt_l$ and proceed similarly as above
\begin{align*}
    \bm{c}_l =  \sum_{i=1}^n  b_i(\xb_i, \yb_i) = N \cdot \mathbb{E}_{i\sim \{1, ..., N \}}[b_i(\xb_i, \yb_i)] \approx  \frac{N}{b} \sum_{i \in \Bar{b}} b_i(\xb_i, \yb_i) = \frac{N}{b} \Kb_{mb} \Tilde{D}_{l, b}^{-1}\ybt_b^l=: \bm{c}^l_b \: ,
\end{align*}
where
\begin{align*}
    b_i(\xb_i, \yb_i)  := \frac{\Tilde{y}_{l,i}}{\Tilde{\sigma}^2_{l,i}}
 \begin{bmatrix}
k_{\theta}(\ub_1, \xb_i) \\[0.6em]
\vdots \\[0.6em]
k(\ub_m, \xb_i) 
\end{bmatrix} \: .
\end{align*}
The estimators for $\bm{\mu}^l_T$ and $\bm{A}^l_T$ are then obtained using a plug-in approach,
\begin{align*}
    \bm{\mu}^l_T = \Kb_{mm} (\bm{\Sigma}^{l})^{-1}\bm{c}_l \approx \Kb_{mm} (\bm{\Sigma}_b^{l})^{-1} \bm{c}^l_b =:  \bm{\mu}_b^l \: ,\\
    \bm{A}_T^l = \Kb_{mm} (\bm{\Sigma}^{l})^{-1} \Kb_{mm} \approx \Kb_{mm} (\bm{\Sigma}_b^{l})^{-1} \Kb_{mm} =: \bm{A}^l _b \: .
\end{align*}
Note that neither of the above estimators is unbiased, since both depend on the inverse $(\bm{\Sigma}_b^{l})^{-1}$. For the empirical investigation of the magnitude of the bias, see Appendix \ref{sec:app_bias}. However, $\bm{A}^l_b$ can be shown to be approximately (up to the first order Taylor approximation) unbiased.
\vspace{10pt}
\begin{proposition}For the estimator $\bm{A}^l_b$ in SVGP-VAE, it holds that
\begin{align*}
    \mathbb{E}[\bm{A}^l_b] - \bm{A}_T^l \approx 0 \: .
\end{align*}
\end{proposition}
\begin{proof}
Note that expectation here is taken with respect to the empirical distribution of the training data, that is, $\mathbb{E}_{i \sim \{1,...,N\}}$. Using the definitions of $\bm{A}^l_b$ and $\bm{A}_T^l$, we get
\begin{align*}
    \mathbb{E}[\bm{A}^l_b] - \bm{A}_T^l = \Kb_{mm} \big(\mathbb{E}\big[(\bm{\Sigma}^l_b)^{-1}\big] - (\bm{\Sigma}^l)^{-1}\big) \Kb_{mm} \: ,
\end{align*}
so it remains to show that $\mathbb{E}[(\bm{\Sigma}^l_b)^{-1}] - (\bm{\Sigma}^l)^{-1} \approx 0$. To this end, we exploit the positive definiteness of the kernel matrix $\Kb_{mm}$ and we approximate both inverse terms with the first order Taylor expansion:
\begin{align*}
    (\bm{\Sigma}^l_b)^{-1} = \big(\Kb_{mm} + \frac{N}{b} \Kb_{mb} \Tilde{D}_{l, b}^{-1} \Kb_{bm}\big)^{-1} = \Kb_{mm}^{-\frac{1}{2}}\big(\Ib + \frac{N}{b}\Kb_{mm}^{-\frac{1}{2}}\Kb_{mb} \Tilde{D}_{l, b}^{-1} \Kb_{bm}\Kb_{mm}^{-\frac{1}{2}}\big)^{-1} \Kb_{mm}^{-\frac{1}{2}} \\ \approx \Kb_{mm}^{-\frac{1}{2}}\big(\Ib - \frac{N}{b}\Kb_{mm}^{-\frac{1}{2}}\Kb_{mb} \Tilde{D}_{l, b}^{-1} \Kb_{bm}\Kb_{mm}^{-\frac{1}{2}}\big) \Kb_{mm}^{-\frac{1}{2}} = \Kb_{mm}^{-1} - \frac{N}{b}\Kb_{mm}^{-1} \Kb_{mb} \Tilde{D}_{l, b}^{-1} \Kb_{bm}\Kb_{mm}^{-1} \: .
\end{align*}
Similarly, we have $(\bm{\Sigma}^l)^{-1} \approx \Kb_{mm}^{-1} - \Kb_{mm}^{-1}\Kb_{mN}\Tilde{D}_l^{-1}\Kb_{Nm}\Kb_{mm}^{-1}$ . Using this, we proceed as
\begin{align*}
    \mathbb{E}[(\bm{\Sigma}^l_b)^{-1}] - (\bm{\Sigma}^l)^{-1} & \approx  - \frac{N}{b}\Kb_{mm}^{-1} \mathbb{E} \big[ \Kb_{mb} \Tilde{D}_{l, b}^{-1} \Kb_{bm} \big]\Kb_{mm}^{-1} + \Kb_{mm}^{-1}\Kb_{mN}\Tilde{D}_l^{-1}\Kb_{Nm}\Kb_{mm}^{-1}  \\
    &= - \frac{N}{b}\Kb_{mm}^{-1} \mathbb{E} \bigg[ \sum_{i \in \Bar{b}} B_i(\xb_i, \yb_i) \bigg]\Kb_{mm}^{-1} + \Kb_{mm}^{-1}\bigg(\sum_{i=1}^N B_i (\xb_i, \yb_i)\bigg)\Kb_{mm}^{-1}  \\ &= - N\Kb_{mm}^{-1} \mathbb{E} \big[  B_i(\xb_i, \yb_i) \big]\Kb_{mm}^{-1} + N \Kb_{mm}^{-1}\mathbb{E} \big[  B_i(\xb_i, \yb_i) \big]\Kb_{mm}^{-1} = 0
\end{align*}
\vspace{1em}
\end{proof}
Note that a similar proof technique unfortunately cannot be used to show that $\bm{\mu}^l_b$ is approximately unbiased for $\bm{\mu}^l_T$, due to the product of two plug-in estimators that both depend on the data in the same batch.

\subsection{Low-rank kernel matrix in \cite{Casale2018GaussianAutoencoders}}
\label{sec:app_low_rank}
In the following, we present an approach from \cite{Casale2018GaussianAutoencoders} to reduce the cubic GP complexity in their GP-VAE model. Note that the exact approach is not given in \cite{Casale2018GaussianAutoencoders} and the derivation shown here is our best attempt at recreating the results.

In \cite{Casale2018GaussianAutoencoders}, datasets composed of $P$ unique objects observed in $Q$ unique views are considered, for instance, images of faces captured from different angles. In total, this amounts to $N = P \cdot Q$ images. The auxiliary data consist of two sets of features $\Xb = \big [\Xb_o \: \Xb_v \big]$, with $\Xb_o \in \R^{N \times p_1}$ containing information about objects (e.g., drawing style of the digit or characteristics of the face) and $\Xb_v \in \R^{N \times p_2}$ containing information about views (e.g., an angle or position in space). Let $\xb_i = [\xb_{o,i}  \: \xb_{v,i}]$ denote auxiliary data for the $i$-th image (corresponding to the $i$-th row of the $\Xb$ matrix). Additionally, denote by $\textbf{P} \in \R^{P \times p_1}$ and $\textbf{Q} \in \R^{Q \times p_2}$ matrices consisting of all unique object and view representations, respectively. A product kernel between a linear kernel for object information and a periodic kernel for view information is used:
\begin{align*}
    k_{\theta}(\xb_i, \xb_j) = \sigma^2 \exp\bigg(-\frac{2\sin ^2 \big(\norm{\xb_{v,i} - \xb_{v,j}}\big)}{ l^2}\bigg) \cdot \xb_{o,i}^T \xb_{o,j} \:, \; \theta = \{\sigma^2, l\} \: .
\end{align*}
Exploiting the product and (partial) linear structure of the kernel and using properties of the Kronecker product, $\Kb_{NN}$ can be written in a low-rank form as
\begin{align*}
    \Kb_{NN}(\Xb, \Xb) = \Pb \Pb^T \otimes \Kb(\textbf{Q}) = \Pb \Pb^T \otimes \Lb \Lb^T = 
    \big(\Pb \otimes \Lb\big) \big( \Pb^T \otimes \Lb^T\big) = \big(\Pb \otimes \Lb\big) \big( \Pb \otimes \Lb\big)^T =: \Vb \Vb^T,
\end{align*}
where $\Kb(\textbf{Q}) \in \R^{Q \times Q}$ is a kernel matrix of all unique view vectors based on the periodic kernel, $\Lb$ is its Cholesky decomposition and $\Vb \in \R^{N \times H}, \: H = Q \cdot p_{1},$ is the obtained low-rank matrix ($H \ll N$ due to the assumption that the number of unique views $Q$ is not large). For such matrices, the inverse and log-determinant can be computed in $O(NH^2)$ using a matrix inversion lemma \citep{henderson1981deriving} and a matrix determinant lemma  \citep{harville1998matrix}, respectively.

While the above approach elegantly reduces the GP complexity for a given dataset (for auxiliary data $\Xb$ with a product structure), it is not readily extensible for other types of datasets (e.g. time series). In contrast, our SVGP-VAE makes no assumptions on neither the data nor the GP kernel used. Therefore, it is a more general solution to scale GP-VAE models.

\subsection{Sparse GP-VAE based on \cite{Titsias2009VariationalProcesses}}
\label{sec:app_one_shot_SVGP-VAE}
Using the sparse GP posterior $q_S$ (Equation \ref{eq:svgp_approx_post}) and ELBO $\ELBO_T$ (Equation \ref{eq:ELBO_T}) from \cite{Titsias2009VariationalProcesses} gives rise to the following sparse GP-VAE ELBO:
\begin{align*}\label{eq:SVGP-VAE_Titsias_ELBO}
    \ELBO_{PT}\big(\Ub, \psi, \phi, \theta):= \sum_{l=1}^L \ELBO_{T}^{l}(\Ub, \, \phi, \, \theta)
  +\sum_{i=1}^N\E_{q_S}  \bigg[  \log \ppsi(\yb_i | \zb_i) - %
     \log \tilde{q}_\phi(\zb_i | \yb_i)\bigg] \: .
\end{align*}
In Section \ref{sec:latent_sparse_GP}, we have outlined how to obtain the above sparse ELBO from the GP-VAE ELBO proposed in \cite{Pearce2019ThePixels}. Alternatively, $\ELBO_{PT}$ can be derived in the standard way by directly considering the KL divergence between the sparse GP posterior and the (intractable) true posterior for the latent variables  $\KL\big(q_S(\Zb|\cdot)||p_{\psi, \theta}(\Zb|\Yb, \Xb)\big) $. 

Following \cite{Titsias2009VariationalProcesses}, we consider the joint distribution of observed and \emph{augmented} latent variables $p_{\psi, \theta}(\Zb, \Fb_m, \Yb | \Xb)$ where 
$\Fb_m := \big[ \fb^1, \dots, \fb^L \big], \; \fb^l := f^l(\Ub) \in  \mathbb{R}^{m}$.
The sparse GP posterior decomposes as $q_S(\Zb, \Fb_m | \cdot) = p_{\theta}(\Zb | \Fb_m) p_S(\Fb_m)$,
where $p_S(\Fb_m) := \prod_{l=1}^L \mathcal{N}(\fb^l_m | \mub^l, \Ab^l)$ is a free variational distribution and $p_{\theta}(\Zb | \Fb_m)$ is a (standard) conditional GP prior. The problem of minimizing the KL divergence is then equivalently posed as a maximization of a lower bound of the model evidence as follows,
where in the first steps we introduce $\Tilde{q}_{\phi}(\Zb | \Yb)$
and $q_S(\Zb, \Fb_m | \cdot)$ and apply Jensen's inequality:
\begin{align*}
    \log p(\Yb | \Xb) 
    & =   \log \int p_{\psi, \theta}(\Zb, \Fb_m, \Yb| \Xb)
    \frac{q_S(\Zb, \Fb_m | \cdot)}{q_S(\Zb, \Fb_m | \cdot)} 
    \frac{\Tilde{q}_{\phi}(\Zb | \Yb)}{\Tilde{q}_{\phi}(\Zb | \Yb)}
    d\Zb d\Fb_m
    \\
    \\
     &\ge   \int q_S(\Zb, \Fb_m | \cdot) \log 
    \frac{p_{\psi, \theta}(\Zb, \Fb_m, \Yb| \Xb)}{q_S(\Zb, \Fb_m | \cdot)} 
    \frac{\Tilde{q}_{\phi}(\Zb | \Yb)}{\Tilde{q}_{\phi}(\Zb | \Yb)}
    d\Zb d\Fb_m  
    \\ \\
    &=
      \int q_S(\Zb, \Fb_m | \cdot) \log \frac{\Tilde{q}_{\phi}(\Zb | \Yb)p_{ \psi}(\Yb| \Zb) p_{\theta}(\Zb | \Fb_m )p_{\theta}(\Fb_m | \Xb)}{\Tilde{q}_{\phi}(\Zb | \Yb)p_{\theta}(\Zb | \Fb_m) p_S(\Fb_m)} d\Zb d\Fb_m 
      \\ \\
     &=\sum_{l=1}^L \int q_S(\zb^l, \fb_m^l | \cdot) \log \frac{\Tilde{q}_{\phi}(\zb^l | \Yb) p_{\theta}(\fb_m^l | \Xb)}{ p_S(\fb_m^l)} d\zb^l d\fb_m^l 
     +
     \sum_{i=1}^N \int q_S(\zb_i | \cdot)\bigg( \log p_{ \psi}(\yb_i| \zb_i) - \log \Tilde{q}_{\phi}(\zb_i | \yb_i) \bigg)d\zb_i  \\ \\[5pt]
      &= \sum_{l=1}^L \mathcal{L}_{T}(\Ub, \phi, \theta, \mub^l, \Ab^l) + \sum_{i=1}^N \mathbb{E}_{q_S}\big[ \log p_{ \psi}(\yb_i| \zb_i) - \log \Tilde{q}_{\phi}(\zb_i | \yb_i) \big]
     \end{align*}
Recall the symmetry of the Gaussian distribution $\tilde{q}_\phi(\zb_i^l|\yb_i)=\mathcal{N}(\zb^l_i|\mub^l(\yb_i), \sigma^l(\yb_i)) = \mathcal{N}(\mub^l(\yb_i)|\zb^l_i, \sigma^l(\yb_i))$. Hence, %
the first term of the penultimate expression is a sum over sparse Gaussian processes, one for each latent channel, and each term is precisely Equation 8 of \citet{Titsias2009VariationalProcesses} for sparse Gaussian process regression.
Therefore we write $\mathcal{L}_T^l$ and let $\mub^l = \mub_T^l$ and $\Ab^l = \Ab^l_T$. For further derivation steps see \cite{Titsias2009VariationalProcesses}.

\newpage

\section{Additional experiments}
\label{sec:app_experiments_chapter}

\subsection{PCA initialization of GP-LVM vectors and inducing points}
\label{app:PCA_init}

In this section, we describe how Principal Component Analysis (PCA) is used to initialize the GP-LVM digit representations as well as the inducing points in the rotated MNIST experiment. Note that both the GP-VAE \citep{Casale2018GaussianAutoencoders} and the SVGP-VAE depend on GP-LVM vectors, with the SVGP-VAE additionally relying on inducing points.

To obtain a continuous digit representations for each digit instance, we start with the data matrix $\Xb \in \R^{P \times K}$ that consists of unrotated MNIST images. PCA is then performed on $\Xb$, yielding a matrix $\bm{D} \in \R^{P \times M}$ whose rows $\bm{d}_i$ are used as initial values for the GP-LVM vectors. $M$ represents the number of principal components kept.

For initialization of the inducing points, we sample $n$ GP-LVM vectors from the empirical distribution based on the PCA matrix $\bm{D}$ for each of the $Q$ angles. This results in a matrix $\Ub_{init} \in \R^{m \times (1 + M)}$ with $m = n \cdot Q$ representing the number of inducing points. The exact procedure is given in Algorithm \ref{alg:PCA_init}. Results from the ablation study on the PCA initialization described here are presented in Table \ref{table:PCA_init_ablation}.

\begin{algorithm}
 \caption{Initialization of inducing points in the SVGP-VAE\label{alg:PCA_init} (rotated MNIST experiment)}
\SetAlgoLined
\SetKwInOut{Input}{input}
\SetKwInOut{Output}{output}
\Input{PCA matrix $\bm{D}$, number of inducing points per angle $n$, set of angles $ \{\frac{2\pi k}{Q} \: | \: k=1,...,Q\}$}
\BlankLine
$\Ub_{init} = [\:] $

\# sample $m=n \cdot Q$ points from empirical distribution of each principle component

 \For{$i=1,...,M$}{
    
    $ \Ub_{init} = \big[\Ub_{init}, \; \emph{sample}\big(\bm{D}[: \: , i], \: nr\_samples=n\big)\big]$

 }
 
 \# add column with angle information
 
 $\bm{a} = \big[\underbrace{2 \pi / Q, ..., 2 \pi / Q}_{n \times}, \: ... \: , \underbrace{2 \pi , ..., 2 \pi}_{n \times} \big]^T \in \R^{m} $
 
 $ \Ub_{init} = \big[\bm{a}, \; \Ub_{init \big]}$
 
 \Return{$\Ub_{init}$}
\end{algorithm}

\begin{table}[H]
\setlength{\tabcolsep}{10pt}
\centering
\begin{tabular}{l l l}
\toprule
 & \textbf{PCA init} & \textbf{random init}  \\
\midrule
  \textbf{GP-VAE} \cite{Casale2018GaussianAutoencoders} & $0.0370 \pm 0.0012$  & $0.0374 \pm 0.0009$ \\[0.08cm]
\textbf{SVGP-VAE} & $0.0251 \pm 0.0005$ &  $0.0272 \pm 0.0006$ \\[0.08cm]
\bottomrule
\end{tabular}
\caption[Rotated MNIST - PCA initialization of GP-LVM vectors and inducing points]{A comparison of different initialization regimes for GP-LVM vectors and inducing points in the rotated MNIST experiment. For random initialization, a Gaussian distribution with mean $0$ and standard deviation $1.5$ was used.}
\label{table:PCA_init_ablation}
\end{table}

\subsection{SVGP-VAE latent space visualization}

In Figure \ref{fig:SVGP-VAE_latents}, we depict two-dimensional t-SNE \citep{vanDerMaaten2008} embeddings of SVGP-VAE latent vectors ($L=16$). Visualized here are latent vectors for training data of the \emph{five-digit} version of the rotated MNIST dataset ($N=20250$). As expected, the model clusters images based on the digit identity. More interestingly, SVGP-VAE also seems to order images within each digit cluster with respect to angles. For example, looking at the cluster of the digit 3 (the blue cluster in the middle of the lower plot), we observe that embeddings of rotated images are ordered continuously from $0$ to $2\pi$ as we move in clockwise direction around the circular shape of the cluster.

\begin{figure}[H]
\centering
\includegraphics[width=\textwidth]{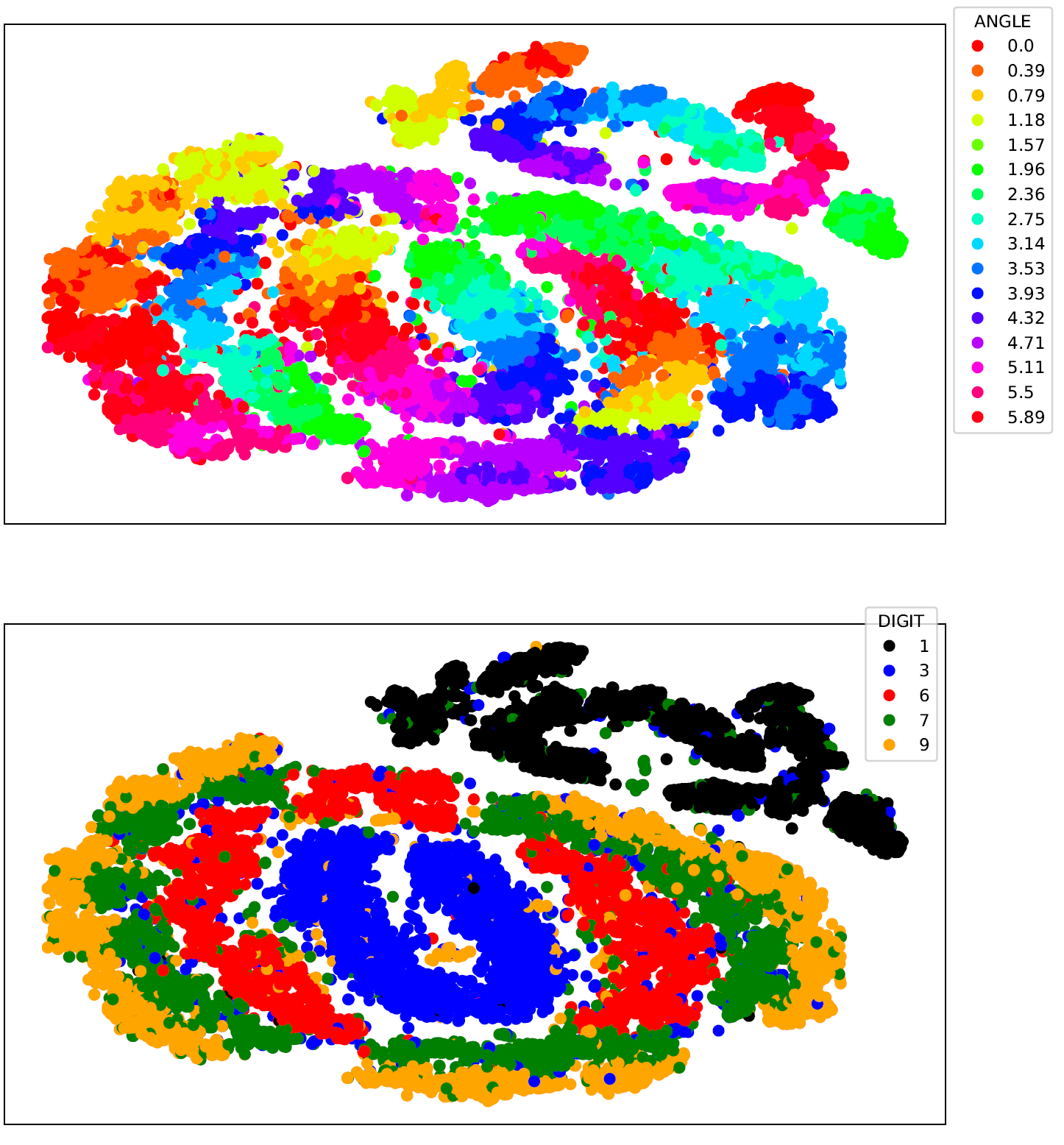}
\caption[Rotated MNIST - SVGP-VAE latent space visualization]{t-SNE embeddings of SVGP-VAE latent vectors on the training data for rotated MNIST. On the upper scatter plot, each image embedding is colored with respect to its associated angle. On the lower scatter plot, each image embedding is colored with respect to its associated digit. The t-SNE perplexity parameter was set to 50.}
\label{fig:SVGP-VAE_latents}
\end{figure}

\subsection{Rotated MNIST: generated images}
\label{sec:app_rot_MNIST_generations}

\begin{figure}[H]
\centering
\includegraphics[width=0.7\textwidth]{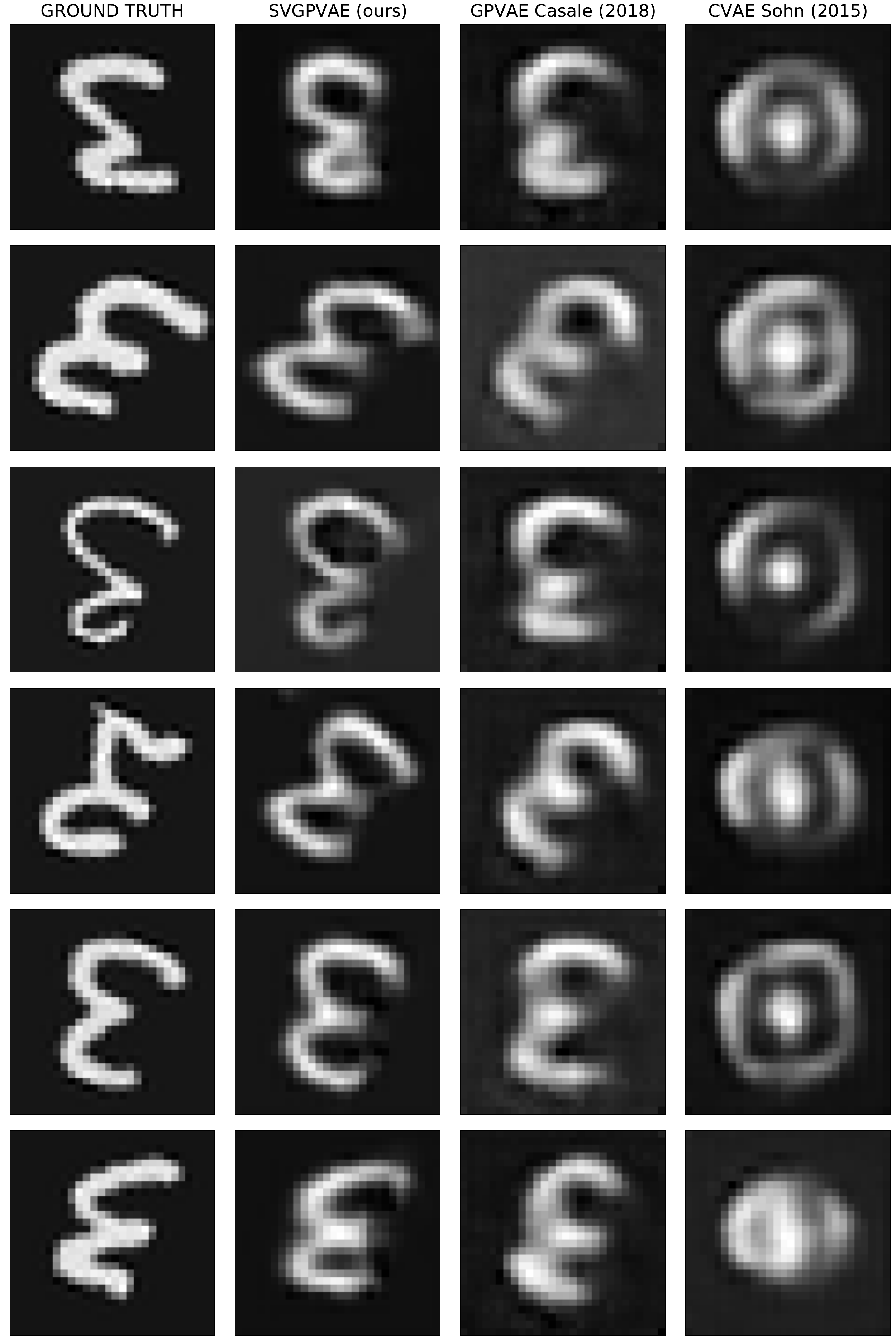}
\caption[Rotated MNIST - generated images 3]{Generated test images in the rotated MNIST experiment for all considered models.}
\label{fig:rot_MNIST_recons_main_3}
\end{figure}

\subsection{Bias analysis of MC estimators in SVGP-VAE}
\label{sec:app_bias}

\begin{figure}
\centering
\includegraphics[width=\textwidth]{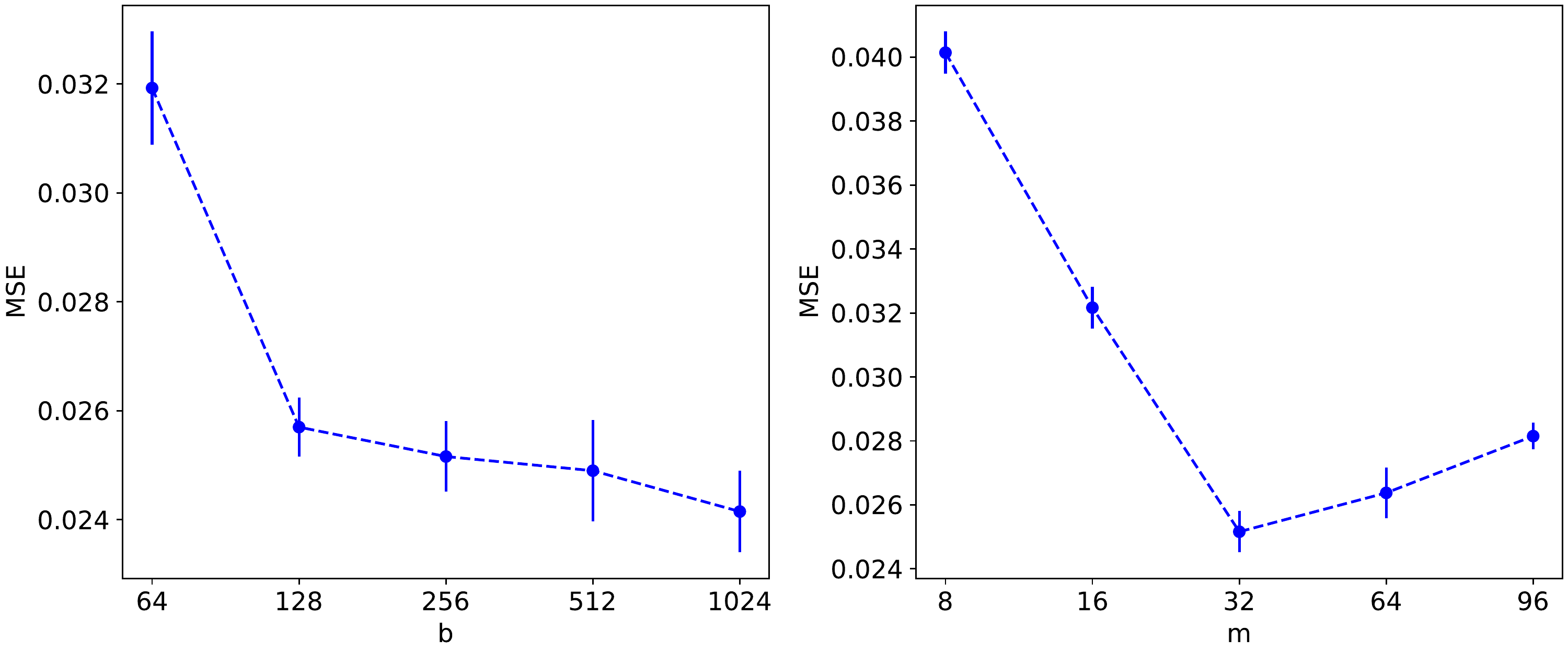}
\caption[Rotated MNIST - SVGP-VAE results for varying batch size and number of inducing point]{SVGP-VAE results on the rotated MNIST dataset (digit 3) for varying batch size (left) and number of inducing points (right). For the batch size experiment, $m$ was set to 32. For the inducing points experiment, $b$ was set to 256. For each configuration, a mean MSE together with a standard deviation based on 5 runs is shown.}
\label{fig:SVGP-VAE_b_M}
\end{figure}

Here we look at some additional experiments that were conducted to get a better understanding of the SVGP-VAE model. Depicted in Figure \ref{fig:SVGP-VAE_b_M} are the results when varying the batch size and the number of inducing points. We first notice that the SVGP-VAE performance improves as the batch size is increased. As pointed out in Section~\ref{sec:best_ELBOs}, this is a consequence of the Monte Carlo estimators from (\ref{eq:MC_estimators}) used in $q_S$ whose quality depends on the batch size. While the dependence on the batch size can surely be seen as one limitation of the model, it is encouraging to see that the model achieves good performance already for a reasonably small batch size (e.g., $b=128$). Moreover, the batch size parameter in the SVGP-VAE offers a simple and intuitive way to navigate a trade-off between performance and computational demands. If one is more concerned regarding the performance, a higher batch size should be used. On the other hand, if one only has limited computational resources at disposal, a lower batch size can be utilized resulting in a faster and less memory-demanding model.

Looking at the plot with the varying number of inducing points next, we observe that the model achieves a solid performance with as little as 16 inducing points on the rotated MNIST data. However, increasing the number of inducing points $m$ starts to have a negative impact on the performance after a certain point. This can be partly attributed to numerical issues that arise during training --- the higher the $m$, the more numerically unstable the inducing point kernel matrix $\Kb_{mm}$ becomes. Moreover, since the number of inducing points equals the dimension of the Monte Carlo estimators in (\ref{eq:MC_estimators}), increasing $m$ results in a larger dimension of the space, potentially increasing the complexity of the estimation problem.

To better understand the effect of the number of inducing points $m$ on the quality of estimation in our proposed MC estimators, we investigate here the trajectory of the bias throughout training. To this end, for each epoch $i$ an estimator $\mub^{l}_{j,i}$ is calculated for each latent channel $l$ and for each batch $j$. Additionally, the true value $\mub_{T,i}^l$ is obtained (based on the entire dataset) for every epoch and every latent channel using model weights from the end of the epoch. The bias for the $l$-th latent channel and $i$-th epoch is then computed as
\begin{align*}
    \bm{b}_i^l := \frac{1}{B}\sum_{j=1}^B \mub^{l}_{j,i} - \mub_{T,i}^l
\end{align*}
where $B := \lceil{\frac{N}{b}}\rceil$ represents the number of batches in a single epoch. Finally, for each epoch $i$ the $L1$ norms of the bias vectors for each latent channel are averaged $\bm{b}_i = \frac{1}{L}\sum_{l=1}^L \bm{b}_i^l$.

Moving averages of the resulting bias trajectories are depicted in Figure \ref{fig:rot_MNIST_bias_SVGP-VAE_vary_m}. For comparison purposes, each trajectory is normalized by the number of inducing points used. Notice how for smaller $m$, the bias trajectories display the expected behavior and converge (or stay close) to 0. Conversely, for larger numbers of inducing points ($m=64$ and $m=96$), the bias is larger and does not decline as the training progresses. This suggests that the proposed estimation might get worse in larger dimensions.

However, despite seemingly deteriorating approximation in higher dimensions, it is also evident that the approximation does not completely break down --- the model still achieves a solid performance even for a larger number of inducing points. Nevertheless, we note that getting a better theoretical grasp of the quality of estimation or reparameterizing the SVGP-VAE ELBO in a way such that these estimators are no longer needed could be a fruitful area of future work.

\begin{figure}[H]
\centering
\includegraphics[width=\textwidth]{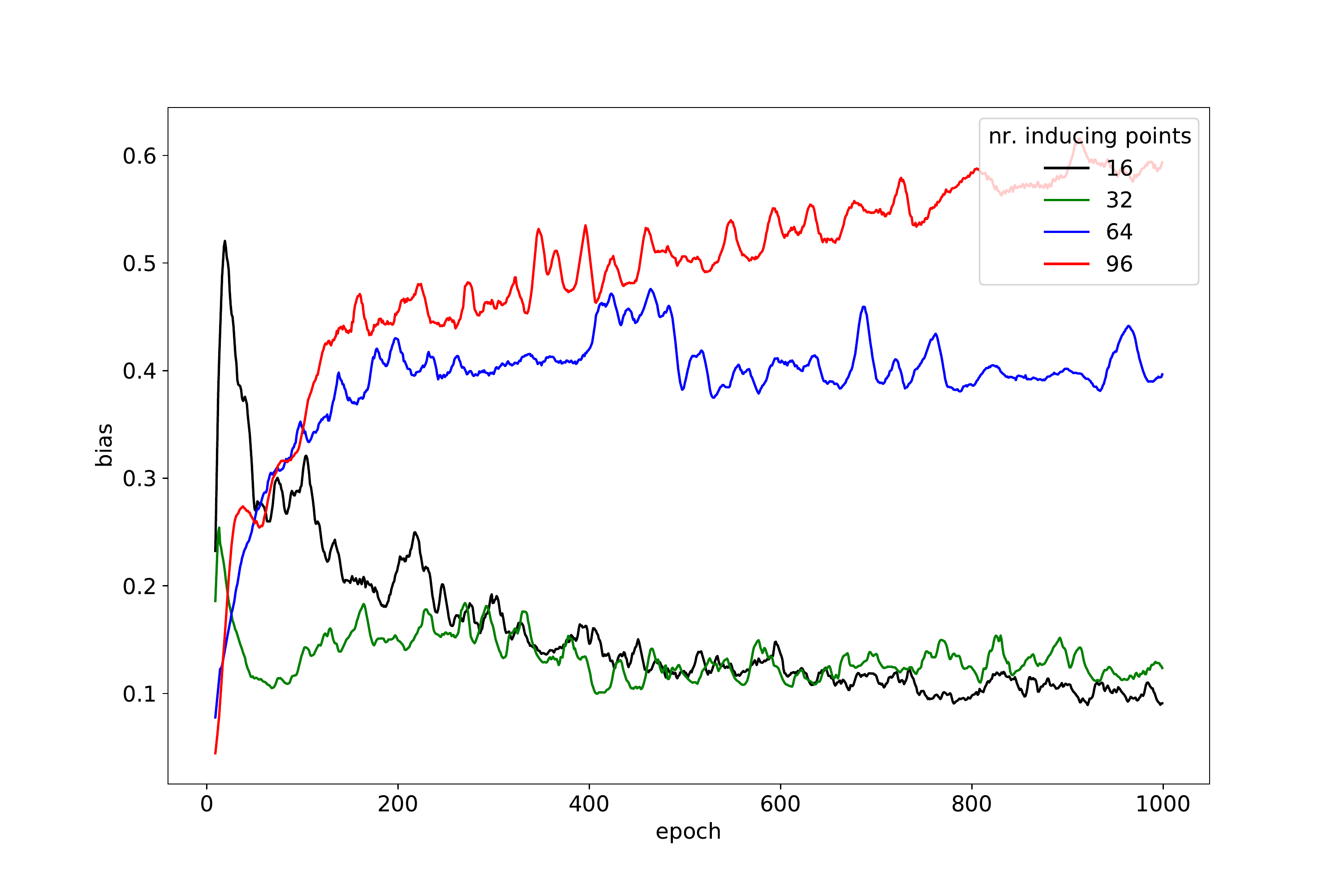}
\caption{Bias trajectories in the SVGP-VAE model for a varying number of inducing points. For all runs, the batch size was set to 256.}
\label{fig:rot_MNIST_bias_SVGP-VAE_vary_m}
\end{figure}

\subsection{Deep sparse GP from \cite{Hensman2013GaussianData} for conditional generation}
\label{sec:app_Hensman_baseline}
In Section \ref{sec:latent_sparse_GP}, we demonstrate that a sparse GP approach from \cite{Hensman2013GaussianData} cannot be used in the GP-VAE framework as it does not lend itself to amortization. In Section \ref{sec:app_vanishing_SVGP-VAE_Hensman}, we then provide a detailed derivation of this phenomenon. Here, we leave out the amortization completely and consider directly the sparse GP from \citet{Hensman2013GaussianData}. To this end, we modify the ELBO in eq. (4) in \cite{Hensman2013GaussianData}. To model our high-dimensional data $\yb_i \in \R^K$, we utilize a deep likelihood parameterized by a neural network $\psi: \R^L \xrightarrow[]{} \R^K$ (instead of a simple Gaussian likelihood). Moreover, we replicate a GP regression L times (across all latent channels), which yields the following objective function

\begin{align*} 
    \ELBO(\Ub, \psi, \theta, \mub^{1:L}, \Ab^{1:L}, \sigma) =  &\sum_{i=1}^N \bigg\{ \log\mathcal{N}\big(\yb_i \:|\: \psi( \bm{m}_i), \: \sigma^{2} \Ib \big)  - \frac{1}{2 \sigma^{2}} \sum_{l=1}^L\: (\Tilde{k}_{ii} + Tr(\textbf{A}^l\: \Lambda_i))  \bigg\} \: - \\& \sum_{l=1}^L \KL\big(q_S^l(\fb_m|\cdot) \: || \: p_\theta(\fb_m|\cdot)\big)
\end{align*}

where $\bm{m}_i := [\bm{k}_i \Kb_{mm}^{-1}\bm{\mu}^1, \: ... \:, \bm{k}_i \Kb_{mm}^{-1}\bm{\mu}^L ]^T \in \R^L$. Also recall that $q^{l}_S(\fb_m | \cdot) = \mathcal{N}(\fb_m | \bm{\mu^l}, \bm{A^l})$, $p_{\theta}(\fb_m | \cdot) = \mathcal{N}(\fb_m | \bm{0}, \: \Kb_{mm})$, $\Lambda_i := \Kb_{mm}^{-1}\bm{k}_i \bm{k}_i^T\Kb_{mm}^{-1}$ and $\Tilde{k}_{ii}$ is the $i$-th diagonal element of $\Kb_{NN} - \Kb_{Nm} \Kb_{mm}^{-1}\Kb_{mN}$.

For a test point $\bm{x_{*}}$, we first obtain $\bm{m}_* = [\bm{k}_* \Kb_{mm}^{-1}\bm{\mu}^1, \: ... \:, \bm{k}_* \Kb_{mm}^{-1}\bm{\mu}^L ]^T, \: \bm{k_*} = [k_{\theta}(\xb_*, \ub_1), ..., k_{\theta}(\xb_*, \ub_m)]^T \in \R^m$, and then pass it through the network $\psi$ to generate $\yb_*$.

For comparison purposes, the same number of latent channels ($L=16$) and the same architecture for the network $\psi$ as in our SVGP-VAE is used. We train this baseline model for 2000 epochs using the Adam optimizer and a batch size of 256.

The strong performance (see Table \ref{table:rot_MNIST_main}) of this baseline provides interesting new insights into the role of amortization in GP-VAE models. For the task of conditional generation, where a single GP prior is placed over the entire dataset, the amortization is not necessary, and one can modify existing sparse GP approaches \citep{Hensman2013GaussianData} to achieve good results in a computationally efficient way. Note that this is not the case for tasks like learning interpretable low-dimensional embeddings \citep{Pearce2019ThePixels} or time-series imputation \citep{Fortuin2019GP-VAE:Imputation}. For such tasks, the inference network is needed in order to be able to quickly obtain predictions for new test points without rerunning the optimization.

More thorough investigation of this baseline, its interpretation, and its comparison to the existing work on deep Gaussian Processes \citep{damianou2013deep, wilson2016deep} is left for future work.

\end{document}